\documentclass{article} 
\usepackage{iclr2027_conference,times}

\usepackage{amsmath,amsfonts,bm}

\def\eqref#1{equation~\ref{#1}}

\def\1{\bm{1}}

\DeclareMathAlphabet{\mathsfit}{\encodingdefault}{\sfdefault}{m}{sl}
\SetMathAlphabet{\mathsfit}{bold}{\encodingdefault}{\sfdefault}{bx}{n}

\usepackage{algorithm}
\usepackage{algorithmic}

\usepackage{amsmath}
\usepackage{amssymb}
\usepackage{amsthm}
\usepackage{bbding} 
\usepackage{booktabs} 
\usepackage{color}
\usepackage{colortbl}
\usepackage{enumitem}

\usepackage{xcolor}
\definecolor{pink}{HTML}{747199}

\usepackage{graphicx}
\usepackage{hyperref}
\hypersetup{colorlinks=true, citecolor=pink, urlcolor=pink, linkcolor=pink}

\usepackage{multicol,multirow}
\usepackage{natbib}
\usepackage{pifont}
\usepackage{setspace}
\usepackage{subfigure}
\usepackage{capt-of}
\usepackage{threeparttable}
\usepackage{url}
\usepackage{wrapfig}

\usepackage{arydshln}
\ADLinactivate

\definecolor{forestgreen}{RGB}{79,173,91}
\definecolor{forestyellow}{RGB}{245,195,66}
\definecolor{myblue}{rgb}{0.82, 0.94, 0.75}
\definecolor{mygold}{rgb}{1, 0.92, 0.56}
\definecolor{mylightblue}{rgb}{0.70, 0.83, 0.96}
\definecolor{mylightyellow}{rgb}{0.96, 0.88, 0.49}
\definecolor{mylightpink}{rgb}{0.93, 0.79, 0.80}
\definecolor{mygray}{HTML}{f0f0f0}

\newcommand{\T}{\mathrm{T}}
\newcommand{\D}{\mathrm{D}}

\theoremstyle{plain}
\newtheorem{theorem}{Theorem}[section]

\newtheorem{lemma}[theorem]{Lemma}

\theoremstyle{definition}
\newtheorem{definition}[theorem]{Definition}

\newcommand{\bst}[1]{{\textbf{\textcolor{red}{#1}}}}
\newcommand{\subbst}[1]{\textcolor{blue}{\underline{{#1}}}}
\newcommand{\reduce}[1]{\textcolor{pink}{#1}}
\newcommand{\scalea}[1]{\scalebox{0.8}{#1}}
\newcommand{\scaleb}[1]{\scalebox{1.0}{#1}}

\newcommand{\ie}{\textit{i}.\textit{e}., }
\newcommand{\eg}{\textit{e}.\textit{g}., }

\usepackage[most]{tcolorbox}

\tcbuselibrary{breakable} 
\tcbuselibrary{skins}
\newtcolorbox[
    use counter=tcboxcounter,number within=section
]{mybox}[3][]{
    left=3pt,
    right=4pt,
    breakable,
    enhanced,
    title=#2 \thetcbcounter: #3,
    #1
}

\title{Deep Time-series Forecasting Needs Kernelized Moment Balancing}

\author{Licheng Pan$^{1,2}$ \quad Haocheng Yang$^3$ \quad Shuting He$^4$ \quad Han Zheng$^1$ \quad Yuzhou Chen$^1$ \\
\textbf{Guangyi Liu}$^1$ \quad \textbf{Kangning Zhang}$^5$  \quad \textbf{Zhixuan Chu}$^{1\,*}$ \quad \textbf{Yuan Lu}$^2$ \quad \textbf{Hao Wang}$^{6}$ \thanks{Corresponding author.}\\
$^1$Zhejiang University \,
$^2$Xiaohongshu Inc \,
$^3$National University of Singapore \\
$^4$Shanghai University of Finance and Economics \,
$^5$Shanghai Jiao Tong University \\
$^6$Mohamed bin Zayed University of Artificial Intelligence \\
\texttt{licheng.pan@zju.edu.cn} \quad \texttt{zhixuanchu@zju.edu.cn} \quad \texttt{haohaow@zju.edu.cn}
}

\iclrfinalcopy
\begin{document}

\maketitle
\lhead{Under review}

\begin{abstract}
Deep time-series forecasting can be formulated as a distribution balancing problem aimed at aligning the distribution of the forecasts and ground truths. According to Imbens’ criterion, balancing two distributions requires that their first moments be identical with respect to any balancing function. We demonstrate that existing objectives enforce moment matching only for one or two predefined balancing functions, thus failing to achieve full distribution balance.
To address this limitation, we propose direct forecasting with kernelized moment balancing (KMB-DF). Unlike existing objectives, KMB-DF adaptively selects the most informative balancing functions from a reproducing kernel Hilbert space (RKHS) to enforce sufficient distribution balancing. A tractable and differentiable formulation is derived to enable estimation from empirical samples and seamless integration into gradient-based training pipelines. Extensive experiments across multiple models and datasets show that KMB-DF consistently improves forecasting accuracy and achieves state-of-the-art performance. Code is available at~\url{https://anonymous.4open.science/r/KMB-DF-403C}.
\end{abstract}

\section{Introduction}

Deep time-series forecasting aims to predict future values from historical observations using neural networks, which supports numerical applications~\citep{tsinghuasurvey}, such as asset pricing in finance~\citep{mars}, user visit prediction in e-commerce~\citep{chen2023mode}, and weather prediction in meteorology~\citep{application_weather2}.
Currently, the progress in this field primarily surrounds two problems~\citep{wang2025nipstimeo1,wang2026iclrqdf}: \textit{(1) the devise of neural architectures serving as the forecast models, and (2) the design of learning objectives driving model training.} Both are essential for improving forecasting performance~\citep{qiudbloss}.

The devise of neural architectures has been widely studied. The key challenge arises from the input autocorrelation, where different steps in input sequence are correlated. To address this challenge, various architectures have been proposed~\citep{Simpletm,tqnet}. A significant advancement lies in Transformer-based models, which employ self-attention mechanisms to model input autocorrelation while scaling effectively to complex tasks~\citep{lin2024cyclenet,itransformer,PatchTST,Time-LLM}. 
Other progresses include recurrent neural networks~\citep{P-sLSTM}, convolutional neural networks~\citep{Moderntcn}, and linear models~\citep{OLinear,DLinear}, each with unique inductive biases and strengths in modeling input autocorrelation.

The design of learning objectives has also garnered significant attention~\citep{qiudbloss,psloss}. The key challenge arises from the label autocorrelation, where different steps in label sequence are correlated, which renders standard objectives like mean squared error (MSE) biased. 
To address this, distribution balancing methods train forecast models by balancing the distributions of forecast and label sequences, ensuring unbiased training despite label autocorrelation~\citep{wang2026iclrdistdf}. However, ideal distribution balance requires identical first moments with respect to any balancing function~\citep{imbens2015causal}. We demonstrate that existing objectives enforce this equivalence for only one or two predefined functions~\citep{wang2025nipstimeo1,wang2025iclrfredf,qiudbloss}, thus failing to achieve the full distribution balance necessary for unbiased forecast model training.

To address the limitation of existing objectives, we propose direct forecasting with kernelized moment balancing (KMB-DF). The core strategy is to adaptively select the most informative balancing functions from a reproducing kernel Hilbert space (RKHS) and align first moments with these functions to define the learning objective. Leveraging the rich capacity of RKHS associated with universal kernels, KMB-DF enforces sufficient distribution balance with theoretical guarantees. Furthermore, the objective formulation is tractable and differentiable, enabling efficient estimation from finite samples and seamless integration into gradient-based training. Experimental results confirm that KMB-DF outperforms existing learning objectives and achieves state-of-the-art performance.

\textbf{Contributions.} The main contributions can be summarized as follows. 
\ding{182} \textbf{We identify a critical limitation of existing learning objectives for time-series forecasting}: they enforce first-moment equivalence for only one or two balancing functions, failing to achieve full distribution balance and thus unbiased training.
\ding{183} \textbf{We propose KMB-DF, a learning objective that aligns first moments with most informative balancing functions.}
This enforces sufficient distribution balance with theoretical guarantees while permitting gradient-based optimization.
\ding{184} \textbf{We evaluate KMB-DF through extensive experiments}, demonstrating that it consistently outperforms existing objectives, generalizes across various forecast models, and achieves state-of-the-art performance.

\section{Preliminaries}
\label{sec:preliminaries}

\subsection{Problem definition}


In this section, we introduce the deep time-series forecasting problem. We use uppercase bold letters (\eg $\mathbf{X}$) to denote matrices, lowercase bold letters (\eg $\mathbf{x}$) to denote vectors, and lowercase normal letters (\eg $x$) to denote scalars.
A time-series consists of a sequence of chronologically ordered observations, denoted as $\mathbf{s}=\{s_1,s_2,...,s_\mathrm{M}\}\in \mathbb{R}^{\mathrm{M}}$, where $\mathrm{M}$ is the total number of observations. Here we adopt the univariate setting, which generalizes naturally to the multivariate setting.

At a given time step $n$, the task definition requires the following components:
(1) the \textit{history sequence} $\mathbf{x} = [s_{n-\mathrm{H}+1}, \ldots, s_n] \in \mathbb{R}^{\mathrm{H}}$, where $\mathrm{H}$ is the history length; (2) the \textit{label sequence} $\mathbf{y} = [s_{n+1}, \ldots, s_{n+\T}] \in \mathbb{R}^{\T}$, where $\T$ is the forecast horizon; (3) the \textit{forecast model} $g$, which is a neural network mapping $\mathbf{x}$ to a forecast sequence $\hat{\mathbf{y}}$. The task is to train a neural network $g$ by optimizing an appropriate learning objective, such that for any given $\mathbf{x}$, the forecast $\hat{\mathbf{y}}$ accurately approximates the ground-truth label sequence $\mathbf{y}$~\citep{wang2025nipstimeo1,tsinghuasurvey}.

\subsection{Learning objectives for deep time-series forecasting}\label{sec:obj}

The learning objective plays a central role in deep time-series forecasting, because it governs how forecasting models are optimized. A widespread objective is mean squared error (MSE)~\citep{FreTS,itransformer,OLinear}, measuring the point-wise divergence between $\mathbf{y}$ and $\hat{\mathbf{y}}$:
\begin{equation}\label{eq:mse}
\mathcal{E}_\mathrm{MSE}=\left\|\mathbf{y}-\hat{\mathbf{y}}\right\|^2_2=\sum_{t=1}^\mathrm{T}\left(y_t-\hat{y}_{t}\right)^2.
\end{equation}
While widely applied in modern literature, $\mathcal{E}_\mathrm{MSE}$ has a critical limitation: it treats future observations as conditionally independent, ignoring the label autocorrelation structure where $y_t$ depends on its predecessors $\mathbf{y}_{<t}$~\citep{DLinear}. Consequently, $\mathcal{E}_\mathrm{MSE}$ constitutes a biased objective. This phenomenon, termed autocorrelation bias, is formalized in Lemma~\ref{lem:bias}.
\begin{lemma}[Autocorrelation bias] \label{lem:bias}
    Let $\mathbf{y}\in\mathbb{R}^\mathrm{T}$ be a label sequence with conditional covariance $\mathbf{\Sigma}\in\mathbb{R}^{\mathrm{T}\times\mathrm{T}}$. $\mathcal{E}_\mathrm{MSE}$ is biased against the likelihood of $\mathbf{y}$ unless $\mathbf{\Sigma}$ is identity, i.e., steps in $\mathbf{y}$ are conditionally independent given $\mathbf{x}$~\citep{wang2025nipstimeo1}.
\end{lemma}

To accommodate label autocorrelation, recent work has developed alternative learning objectives along two main lines. The first applies a transform function $\phi$ to obtain latent components, $\phi(\hat{\mathbf{y}})$ and $\phi(\mathbf{y})$, and aligns them component-wise~\citep{wang2025iclrfredf,qiudbloss,wang2025nipstimeo1}.
The second line seeks to align the shapes of $\mathbf{y}$ and $\hat{\mathbf{y}}$, motivated by the observation that label autocorrelation is reflected in the temporal structure of the label sequence. This line can be traced back to dynamic time warping (DTW), which measures shape dissimilarity through the minimum-cost alignment of two sequences, but its non-differentiability limits its use in training neural networks. To address this, SoftDTW~\citep{soft-dtw} provides a differentiable relaxation and has been applied across domains~\citep{softdtw_geoscience,softdtw_health2}. Building on SoftDTW, \citet{Dilate} incorporate a temporal distortion index to penalize lag mismatches, while \citet{STRIPE2} extend the approach to probabilistic forecasting.
\textit{Despite the empirical effectiveness of the works described above, they fail to guarantee theoretical unbiasedness as learning objectives under label autocorrelation. }For transform-based objectives, Lemma~\ref{lem:bias} establishes unbiasedness only when the resulting components are conditionally decorrelated, a condition that existing transform functions do not guarantee~\citep{wang2026iclrqdf}. For shape-based objectives, they do not provide theoretical guarantee of unbiasedness under autocorrelated labels\footnote{As a growing literature, there are other learning objectives for training deep forecast models~\citep{tdalign,psloss,timesql,AST}. While they improve upon $\mathcal{E}_\mathrm{MSE}$, they are not explicitly designed to accommodate label autocorrelation~\citep{TimeGrad,MoLA} central to this paper. Thus, we list them for completeness but omit further discussion for brevity.}.

\subsection{Distribution balancing and Imbens' criterion}
Suppose $\mathbb{P}_1$ and $\mathbb{P}_2$ are two distributions, ideal distribution balance is attained when the two distributions are indistinguishable with respect to all relevant statistics. This condition is characterized by Imbens' criterion~\citep{imbens2015causal} as:
\begin{equation}\label{eq:imbens}
    \forall \phi: \mathbb{E}_{z\in\mathbb{P}_1}\left[\phi(z)\right]=\mathbb{E}_{z\in\mathbb{P}_2}\left[\phi(z)\right],
\end{equation}
where $\phi$ is an arbitrary balancing function. This equality is required to hold for any choice of $\phi$, which ensures that $\mathbb{P}_1$ and $\mathbb{P}_2$ have equivalent first moment under any transformation. This criterion is widely used to assess distribution balancing in diverse fields, such as causal inference~\citep{cbps,cbgps} and trustworthy machine learning~\citep{li2023propensity}.



\section{Methodology}
\subsection{Motivation}\label{subsec:motivation}
The training of deep time-series forecast models can be interpreted as a distribution balancing problem~\citep{wang2026iclrdistdf}. The forecast model is considered accurately trained if $\mathbb{P}(\mathbf{y}|\mathbf{x})$ and $\mathbb{P}(\hat{\mathbf{y}}|\mathbf{x})$ are balanced, as it correctly captures the underlying generation process of the label sequence. 
To achieve distribution balancing, likelihood maximization offers a practical way, exemplified by $\mathcal{E}_{\mathrm{MSE}}$ and its variants~\citep{wang2025iclrfredf,wang2025nipstimeo1,wang2026iclrqdf}. However, they suffer from autocorrelation bias in Lemma~\ref{lem:bias}, thereby failing to ensure distribution balancing and theoretical unbiasedness.

\begin{wraptable}{r}{0.5\linewidth}
\vspace{-2em}
\centering\scriptsize
\caption{The choice of $\phi$ in recent learning objectives for deep time-series forecasting.}
\label{tab:phi}
\setlength\tabcolsep{3.2pt}
\begin{tabular}{llll}
\toprule
    Method & Venue & Choice of $\phi$ & \# $\phi$\\
    \midrule
    FreDF & ICLR-25 & Discrete Fourier transform. & 1 \\
    Time-o1 & NeurIPS-25 & Principal component analysis. & 1 \\
    DBLoss & NeurIPS-25 & Trend-seasonal EMA decomposition. & 2 \\
    QDF & ICLR-26 & Meta-learned linear mapping. & 1 \\
    DistDF & ICLR-26 & Mean and variance statistics. & 2\\
\bottomrule
\end{tabular}

\vspace{-9pt}
\end{wraptable}

Imbens’ criterion characterizes distributional balance through Eq.~(\ref{eq:imbens}), providing a formal basis for assessing the limitations of existing learning objectives. As shown in Table~\ref{tab:phi}, these objectives seek to balance distributions with respect to a restricted set of balancing functions $\phi$. For example, $\mathcal{E}_{\mathrm{MSE}}$ seeks to satisfy Eq.~(\ref{eq:imbens}) with the identity map, whereas FreDF uses the Fourier transform. However, each objective considers only \textbf{one or two prespecified} balancing functions, while Imbens’ criterion requires Eq.~(\ref{eq:imbens}) to hold for \textbf{all} admissible $\phi$. Consequently, optimizing these objectives is insufficient to guarantee full distribution balance for training forecasting models. 

The promise of distribution balancing, together with the limitations of existing learning objectives, motivates objectives that explicitly target distribution balance as defined by Imbens' criterion. This raises three research questions: \textit{(1) How can Imbens' criterion be used to formulate a learning objective for training forecasting models? (2) Can such an objective incorporate a broader class of balancing functions? (3) Does it improve forecasting performance in practice?}


\subsection{Distribution balancing with specified $\phi$}
In this section, we develop a distribution balancing strategy to train forecast models. Notably, directly balancing the conditional distributions $\mathbb{P}(\mathbf{y}|\mathbf{x})$ and $\mathbb{P}(\hat{\mathbf{y}}|\mathbf{x})$ is intractable given finite-sample time-series datasets\footnote{Specifically, for any fixed $\mathbf{x}$, the dataset contains only one associated label sequence, and the model produces only one forecast sequence. Thus, only one sample can be observed from each of $\mathbb{P}(\mathbf{y}|\mathbf{x})$ and $\mathbb{P}(\hat{\mathbf{y}}|\mathbf{x})$, which is insufficient to characterize—and therefore balance—these conditional distributions.}. We instead balance the joint distributions $\mathbb{P}(\mathbf{x},\mathbf{y})$ and $\mathbb{P}(\mathbf{x},\hat{\mathbf{y}})$, which is effective for two reasons \citep{wang2026iclrdistdf}:
(i) $\mathbb{P}(\mathbf{x},\mathbf{y})=\mathbb{P}(\mathbf{x},\hat{\mathbf{y}})$ implies $\mathbb{P}(\mathbf{y}|\mathbf{x})=\mathbb{P}(\hat{\mathbf{y}}|\mathbf{x})$ for all $\mathbf{x}$, thus ensuring effective training of forecast models;
(ii) joint balancing is tractable since the joint distributions can be estimated and balanced using all dataset samples.

Let $\mathbf{z}=(\mathbf{x},\mathbf{y})$ and $\hat{\mathbf{z}}=(\mathbf{x},\hat{\mathbf{y}})$ denote the joint random variables. According to Imbens' criterion in Eq.(\ref{eq:imbens}), the two joint distributions are balanced if we have
\begin{equation}\label{eq:imbensjoint}
    \forall \phi: \mathbb{E}_{\mathbf{z}\in\mathbb{P}(\mathbf{z})}\left[\phi(\mathbf{z})\right]=\mathbb{E}_{\hat{\mathbf{z}}\in\mathbb{P}(\hat{\mathbf{z}})}\left[\phi(\hat{\mathbf{z}})\right].
\end{equation}
Given a finite time-series dataset $\{\mathbf{z}_n=(\mathbf{x}_n,\mathbf{y}_n)\}_{n=1}^\mathrm{N}$ and corresponding forecasts $\{\hat{\mathbf{z}}_n=(\mathbf{x}_n,\hat{\mathbf{y}}_n)\}_{n=1}^\mathrm{N}$ where $\hat{\mathbf{y}}_n=g_\theta(\mathbf{x}_n)$. One can attempt to enforce Eq.(\ref{eq:imbensjoint}) over a chosen set of $\mathrm{K}$ balancing functions $\{\phi_k\}_{k=1}^\mathrm{K}$. This yields the constrained learning problem as follow:
\begin{equation}\label{eq:constrained_balance}
    \min_{\theta} \sum_{n=1}^\mathrm{N} \left\|\mathbf{y}_n-\hat{\mathbf{y}}_n\right\|_2^2,\quad
    \text{s.t.} \,
    \sum_{n=1}^\mathrm{N}\phi_k(\mathbf{z}_n)=\sum_{n=1}^\mathrm{N}\phi_k(\mathbf{\hat{z}}_n), \forall k \in [\mathrm{K}]
\end{equation}
where each constraint enforces balance with respect to certain balancing function $\phi_k$. In contrast to the objectives in Table~\ref{tab:phi}, the formulation in Eq.(\ref{eq:constrained_balance}) increases the number of balancing functions to $\mathrm{K}$. As $\mathrm{K}$ grows, the constraints in Eq.(\ref{eq:constrained_balance}) increasingly guarantee the Imbens' criterion, thereby promoting distribution balance between $\mathbb{P}(\mathbf{z})$ and $\mathbb{P}(\hat{\mathbf{z}})$.

However, the formulation Eq.(\ref{eq:constrained_balance}) relies on a predefined set of balancing functions $\{\phi_k\}_{k=1}^\mathrm{K}$. In practice, selecting a finite set of balancing functions that sufficiently depicts distribution discrepancies is difficult. To guarantee Eq.(\ref{eq:imbensjoint}), one would need to specify infinitely many balancing constraints, rendering Eq.(\ref{eq:constrained_balance}) computationally intractable. Therefore, it is essential to develop tractable relaxations that selectively enforce balance over the most informative balancing functions.

\subsection{Kernel function, universal property, \& kernelized soft-margin balancing}
In this section, we leverage reproducing kernel Hilbert space (RKHS) theory to construct a rich, data-adaptive, and theoretically grounded set of balancing functions and derive a tractable learning problem, thereby overcoming the reliance on a predefined finite set of balancing functions. We begin by clarifying key properties of kernel functions in Definitions \ref{def:kernel} and \ref{def:univ}.

\begin{definition}[Kernel function]\label{def:kernel}
Let $\mathcal{Z}$ be a non-empty set. A function $\mathcal{K}: \mathcal{Z} \times \mathcal{Z} \rightarrow \mathbb{R}$ is a kernel function if there exist a Hilbert space $\mathcal{H}$ and a feature map $\psi: \mathcal{Z} \rightarrow \mathcal{H}$ such that $\forall \,\mathbf{z}, \mathbf{z}^{\prime} \in \mathcal{Z}$,
$\mathcal{K}(\mathbf{z}, \mathbf{z}^{\prime}):=\left\langle\psi(\mathbf{z}), \psi(\mathbf{z}^{\prime})\right\rangle_{\mathcal{H}}$.
\end{definition}

\begin{definition}[Universal kernel]\label{def:univ}
    For $\mathcal{Z}$ a compact Hausdorff space, a universal kernel ensures that any continuous function  $e: \mathcal{Z} \to \mathbb{R}$  can be approximated arbitrarily well within its RKHS. Specifically, for any  $\epsilon > 0$, there exists  $f \in \mathcal{H}$  such that: $\sup _{\mathbf{z} \in \mathcal{Z}}|f(\mathbf{z})-e(\mathbf{z})| \leq \epsilon$.
\end{definition}

The exponential kernel, $\mathcal{K}(\mathbf{z}, \mathbf{z}^{\prime}) = \exp(-\tfrac{|\mathbf{z}-\mathbf{z}^{\prime}|}{2\sigma^2})$, is universal in the sense of Definition~\ref{def:univ}.
It implies that functions in the RKHS of $\mathcal{K}$, defined as $\mathcal{H} = \operatorname{span} \{\mathcal{K}(\cdot, \mathbf{z}) \mid \mathbf{z} \in \mathcal{Z} \}$, can approximate any continuous balancing function. Motivated by this property, it is natural to instantiate the balancing functions as exponential kernel functions anchored by $\{\mathbf{z}_k\}_{k=1}^\mathrm{K}$:
\begin{align*}
    \phi_k(\cdot)=\mathcal{K}(\cdot, \mathbf{z}_k),\quad \text{i.e.,} \, \forall \mathbf{z}, \phi_k(\mathbf{z})=\mathcal{K}(\mathbf{z}, \mathbf{z}_k),
\end{align*}
which yields $\mathrm{K}$ balancing functions.
On this basis, substituting $\phi_k(\cdot)=\mathcal{K}(\cdot, \mathbf{z}_k)$ into Eq.(\ref{eq:constrained_balance}) yields the following kernelized constraints that are directly computable from the data:
\begin{equation}\label{eq:constrained_balance_kernel}
    \min_{\theta} \sum_{n=1}^\mathrm{N} \left\|\mathbf{y}_n-\hat{\mathbf{y}}_n\right\|_2^2,\quad
    \text{s.t.}\,
    \sum_{n=1}^\mathrm{N}\mathcal{K}(\mathbf{z}_n,\mathbf{z}_k)=\sum_{n=1}^\mathrm{N}\mathcal{K}(\hat{\mathbf{z}}_n,\mathbf{z}_k), \forall k \in [\mathrm{K}]
\end{equation}
One could directly set $\mathrm{K}=\mathrm{N}$, \ie using balancing functions anchored at all samples.
However, the resulting $\mathrm{N}$ constraints are computationally prohibitive and can over-constrain large time-series datasets.
To obtain a tractable formulation, two refinements to Eq.(\ref{eq:constrained_balance_kernel}) are introduced as follows.
\begin{itemize}[leftmargin=*]
    \item First, we introduce a selective balancing mechanism. 
    Noticing that different balancing functions have different abilities to detect imbalance, we quantify the informativeness score of $\phi_k$ as:
    \begin{equation}\label{eq:delta}
        \delta_k =\sum_{n=1}^\mathrm{N}\mathcal{K}(\mathbf{z}_n,\mathbf{z}_k)-\sum_{n=1}^\mathrm{N}\mathcal{K}(\hat{\mathbf{z}}_n,\mathbf{z}_k), \quad\forall k \in [\mathrm{N}]
    \end{equation}
    where larger $|\delta_k|$ signifies greater imbalance. We select only $\mathrm{K}$ ($\mathrm{K}<\mathrm{N}$) balancing functions with the largest $|\delta_k|$, and denote their associated anchor samples as $\{\mathbf{z}_{n_k}\}_{k=1}^\mathrm{K}$. By prioritizing the most informative balancing functions, we reduce the number of constraints while maintaining effective distribution balance. This strategy is theoretically grounded in Theorem~\ref{thm:kernel_align}, which guarantees distribution alignment between $\mathbb{P}(\mathbf{z})$ and $\mathbb{P}(\hat{\mathbf{z}})$ provided that the optimal discrimination function lies within the linear span of the selected balancing functions $\{\mathcal{K}(\cdot,\mathbf{z}_{n_k})\}_{k=1}^\mathrm{K}$.
    \item Second, we introduce a soft-margin relaxation mechanism. Since the equality constraints in Eq.(\ref{eq:constrained_balance_kernel}) might be overly rigid—particularly for time-series data, where outliers or abrupt shocks may induce large but spurious imbalances—we replace each equality constraint with a pair of inequality constraints that allow bounded violations, and penalize these violations in the objective.
\end{itemize} 
\begin{theorem}
    \label{thm:kernel_align}
    Suppose $\mathcal{H}$ is the RKHS of a exponential kernel $\mathcal{K}$, $f^* \in \mathcal{H}$ is the discrimination function that maximizes the discrepancy between two distributions $\mathbb{P}$ and $\mathbb{Q}$. If $f^*$ lies within the linear span of a finite kernel function set: $f^*\in \{\mathcal{K}(\cdot, \mathbf{z}_k)\}_{k=1}^\mathrm{K}$; then $\mathbb{P}$ and $\mathbb{Q}$ are balanced if $\forall k \in [\mathrm{K}]:$ $\mathbb{E}_{\mathbf{z}\in\mathbb{P}}[\mathcal{K}(\mathbf{z}, \mathbf{z}_k)] = \mathbb{E}_{\hat{\mathbf{z}}\in\mathbb{Q}}[\mathcal{K}(\hat{\mathbf{z}}, \mathbf{z}_k)]$.
\end{theorem}

The two efforts collectively yield the following kernelized soft-margin balancing problem:
\begin{equation}\label{eq:constrained_balance_kernel2}
\begin{aligned}
\min_{\theta,\boldsymbol{\varepsilon}} \quad
& \sum_{n=1}^{\mathrm{N}}
\left\|\mathbf{y}_n-\hat{\mathbf{y}}_n\right\|_2^2
+\kappa\sum_{k=1}^{\mathrm{K}}\varepsilon_k \\
\text{s.t.} \quad
& \left|
\sum_{n=1}^{\mathrm{N}}
\left[
\mathcal{K}(\mathbf{z}_n,\mathbf{z}_{n_k})
-\mathcal{K}(\hat{\mathbf{z}}_n,\mathbf{z}_{n_k})
\right]
\right|
\leq C+\varepsilon_k,
\quad k=1,\ldots,\mathrm{K},\\
& \varepsilon_k\geq 0,
\quad k=1,\ldots,\mathrm{K},
\end{aligned}
\end{equation}
where $C\geq 0$ specifies tolerated imbalance, $\kappa>0$ controls the penalty for exceeding this tolerance, and $\varepsilon_k$ are slack variables. This formulation preserves the balancing principle underlying Imbens' criterion while offering a computationally tractable objective given large-scale time-series~data.

The formulation in Eq.~(\ref{eq:constrained_balance_kernel2}) matches moments for $\mathrm{K}$ adaptively selected balancing functions. This provides a tractable approximation to Imbens' criterion, though balancing a finite set does not generally guarantee balance over all admissible functions. \textit{Nevertheless, relative to existing objectives that employ at most two fixed balancing functions, it improves by incorporating a broader, data-adaptive set of balancing functions}---a significant step forward. Moreover, Theorem~\ref{thm:kernel_align} demonstrate that ideal distribution balancing can be achieved when $f^*$ lies in the span of the selected functions.

\subsection{Model implementation}

In this section, we introduce the implementation of KMB-DF, which adapts the kernelized soft-margin balancing problem to train a deep forecast model $g_\theta$, as detailed in Algorithm~\ref{algo:kmbdf}.

\begin{figure}[t]
\vspace{-2em}
\begin{algorithm}[H]
\caption{The workflow of KMB-DF.}
\label{algo:kmbdf}
\setstretch{0.8}
\footnotesize
\flushleft
\textbf{Input}: $\mathbf{X}$: history sequences, $\mathbf{Y}$: label sequences. \\
\textbf{Parameter}: 
$\alpha$: the strength of penalty, $\mathrm{K}$: the number of balancing functions, $\mathrm{C}$: the tolerated imbalance, $g_\theta$: the forecast model with parameters $\theta$. \\
\textbf{Output}: $\mathcal{L}_\mathrm{KMB-DF}$: the obtained learning objective.\\
\begin{algorithmic}[1] 
\STATE $\hat{\mathbf{Y}}\leftarrow g_\theta(\mathbf{X})$ \COMMENT{{Forecast sequence generation}}
\STATE $\mathbf{Z}=\mathrm{concat}(\mathbf{X},\mathbf{Y})$, $\hat{\mathbf{Z}}=\mathrm{concat}(\mathbf{X},\hat{\mathbf{Y}})$ \COMMENT{{Extended sequence construction}}
\STATE $\delta_k=\sum_{n=1}^\mathrm{N} \left[\mathcal{K}(\mathbf{z}_n,\mathbf{z}_k)-\mathcal{K}(\hat{\mathbf{z}}_n,\mathbf{z}_k)\right]$, $k=1,...,\mathrm{N}$ \COMMENT{{Informativeness score computation}}
\STATE $n_1,...,n_\mathrm{K}\leftarrow \arg\mathrm{TopK}_{k}\delta_k$ \COMMENT{{Balancing function selection}}
\STATE $\xi_{n_k} = [-\mathrm{C}-\delta_{n_k}]_++[\delta_{n_k}+\mathrm{C}]_+$, $k=1,...,\mathrm{K}$  \COMMENT{{Penalty computation}}
\STATE $\mathcal{E}_\mathrm{KMB-DF}\leftarrow\alpha\sum\limits_{k=1}^\mathrm{K} \xi_{n_k}+ \left(1-\alpha\right)\sum\limits_{n=1}^\mathrm{N} \left\|\mathbf{y}_n-\hat{\mathbf{y}}_n\right\|_2^2$ \COMMENT{{Learning objective computation}}
\end{algorithmic}
\end{algorithm}

\vspace{-2em}
\end{figure}

Given a set of history sequences $\mathbf{X}\in \mathbb{R}^{\mathrm{N} \times \mathrm{H}}$ and corresponding label sequences $\mathbf{Y}\in \mathbb{R}^{\mathrm{N} \times \mathrm{T}}$, the forecast model generates the forecast sequences (step 1). To perform joint distribution balancing, we construct the extended sequences (step 2), where $\mathrm{concat}(\cdot,\cdot)$ denotes concatenation along the time dimension. Then, we compute informativeness scores for all balancing functions and retain the top $\mathrm{K}$ functions with largest $|\delta_k|$ (steps 3-4). The final learning objective is:
\begin{equation}\label{eq:obj}
\begin{aligned}
    \mathcal{E}_\mathrm{KMB-DF} &=\alpha\sum_{k=1}^\mathrm{K} \xi_{n_k}+ \left(1-\alpha\right)\sum_{n=1}^\mathrm{N} \left\|\mathbf{y}_n-\hat{\mathbf{y}}_n\right\|_2^2,\\
    \text{where}\;\; \xi_{n_k} &= [-\mathrm{C}-\delta_{n_k}]_++[\delta_{n_k}+\mathrm{C}]_+, \delta_{n_k}=\sum_{n=1}^\mathrm{N} \left[\mathcal{K}(\mathbf{z}_n,\mathbf{z}_{n_k})-\mathcal{K}(\hat{\mathbf{z}}_n,\mathbf{z}_{n_k})\right],
\end{aligned}
\end{equation}
where $[\cdot]_+=\max(\cdot,0)$, $\alpha\in [0,1]$ controls the strength of penalty for violating balancing constraints, and $\mathrm{C}$ controls the tolerance margin for imbalance. This unconstrained objective is directly obtained from Eq.(\ref{eq:constrained_balance_kernel2}) by introducing the hinge penalties $\xi_k$, enabling gradient-based optimization (step 6).
By iterative minimization, KMB-DF improves forecast performance by effectively balancing joint distributions while preserving the benefits of the canonical DF framework~\citep{DLinear}, such as efficient multi-task training and inference. Moreover, as a model-agnostic learning objective, KMB-DF can be integrated into various forecast models to improve performance.

\section{Experiments}
In this section, we conduct an empirical investigation to demonstrate the utility of KMB-DF in training forecast models across five main aspects as follows:
\begin{enumerate}[leftmargin=*]
    \item \textbf{Performance:} \textit{Does KMB-DF perform well?}  In Section~\ref{sec:overall}, we compare KMB-DF with previously developed learning objectives for training forecast models.
    \item \textbf{Gain:} \textit{Why does it work?} In section \ref{sec:ablation}, we perform an ablative study, dissecting the individual components and clarifying their contributions to forecast accuracy.
    \item \textbf{Generality:} \textit{Does it support other models and kernels?} In Section \ref{sec:kernel_studies}, we analyze the impact of kernel selection, and in Section~\ref{sec:generalize}, we examine its compatibility with various models.
    \item \textbf{Sensitivity:} \textit{Is it sensitive to hyperparameters?} In Section \ref{sec:hyper}, we analyze the sensitivity to the hyperparameters $\alpha$, $\mathrm{K}$ and $\mathrm{C}$, showing stable performance across a broad parameter range.
    \item \textbf{Efficiency:} \textit{Is it computationally efficient?} In Appendix, we evaluate the running cost of KMB-DF across different scenarios.
\end{enumerate}

\subsection{Experimental Setup}
\textbf{Datasets.} 
We evaluate our methods using several standard public benchmarks for long-term time-series forecasting, following~\citet{Timesnet}. Specifically, we use the ETT (four subsets), ECL, and Weather~\citep{itransformer} datasets. Additionally, we introduce M5 dataset~\citep{M5}, which is notably challenging. All datasets are split chronologically into training, validation, and test sets. Comprehensive dataset statistics are presented in Appendix~\ref{sec:dataset}.


\textbf{Baselines.} Since this paper focuses on the devise of learning objectives, we select competitive learning objectives tailored for training forecast models as baselines following codebase of~\citet{wang2026iclrdistdf}: \ding{182} \textbf{shape-alignment objectives:} GDTW~\citep{GDTW}, Dilate~\citep{Dilate}, and Soft-DTW~\citep{soft-dtw}; \ding{183} \textbf{likelihood maximation objectives:} QDF~\citep{wang2026iclrqdf}, Time-o1~\citep{wang2025nipstimeo1}, Koopman~\citep{koopman}, FreDF~\citep{wang2025iclrfredf}, and MSE; \ding{184} \textbf{component-wise alignment objectives:} 
DBLoss~\citep{qiudbloss}; \ding{185} \textbf{distribution balancing objectives:} DistDF~\citep{wang2026iclrdistdf}. 

\textbf{Implementation.}  We employ CFPT~\citep{CFPT} as the default forecast model in comparing learning objectives. To ensure fair comparison, the drop-last trick is disabled as recommended in~\citet{qiutfb}. All objectives are trained with the Adam optimizer~\citep{Adam}. When integrating KMB-DF to train a forecast model, we retain all hyperparameters from the public benchmarks~\citep{CFPT}, only tuning $\alpha$, the learning rate, the tolerated imbalance $\mathrm{C}$ and the number of balancing functions $\mathrm{K}$. Experiments are run on Intel(R) Xeon(R) Platinum 8463B CPUs with 32 NVIDIA RTX H800 GPUs. More implementation details are provided in Appendix~\ref{sec:reproduce}.

\begin{table}[t]
\centering
\begin{threeparttable}
\caption{Comparative results with other objectives for time-series forecasting.}\label{tab:loss_avg}
\renewcommand{\arraystretch}{0.8}
\setlength{\tabcolsep}{2pt}
\scriptsize
\renewcommand{\multirowsetup}{\centering}
\begin{tabular}{c|c|cc|cc|cc|cc|cc|cc|cc|cc|cc|cc|cc}
    \toprule
    \multicolumn{2}{l}{\rotatebox{0}{\scaleb{Loss}}} & 
    \multicolumn{2}{c}{\rotatebox{0}{\scaleb{\textbf{KMB-DF}}}} &
    \multicolumn{2}{c}{\rotatebox{0}{\scaleb{QDF}}} &
    \multicolumn{2}{c}{\rotatebox{0}{\scaleb{DistDF}}} &
    \multicolumn{2}{c}{\rotatebox{0}{\scaleb{Time-o1}}} &
    \multicolumn{2}{c}{\rotatebox{0}{\scaleb{FreDF}}} &
    \multicolumn{2}{c}{\rotatebox{0}{\scaleb{DBLoss}}} &
    \multicolumn{2}{c}{\rotatebox{0}{\scaleb{Koopman}}} &
    \multicolumn{2}{c}{\rotatebox{0}{\scaleb{GDTW}}} &
    \multicolumn{2}{c}{\rotatebox{0}{\scaleb{Dilate}}} &
    \multicolumn{2}{c}{\rotatebox{0}{\scaleb{Soft-DTW}}} &
    \multicolumn{2}{c}{\rotatebox{0}{\scaleb{MSE}}} \\
    \cmidrule(lr){3-4} \cmidrule(lr){5-6}\cmidrule(lr){7-8} \cmidrule(lr){9-10}\cmidrule(lr){11-12} \cmidrule(lr){13-14} \cmidrule(lr){15-16} \cmidrule(lr){17-18} \cmidrule(lr){19-20} \cmidrule(lr){21-22} \cmidrule(lr){23-24} 
    \multicolumn{2}{l}{\rotatebox{0}{\scaleb{Metrics}}}  & \scalea{MSE} & \scalea{MAE}  & \scalea{MSE} & \scalea{MAE}  & \scalea{MSE} & \scalea{MAE}  & \scalea{MSE} & \scalea{MAE}  & \scalea{MSE} & \scalea{MAE}  & \scalea{MSE} & \scalea{MAE} & \scalea{MSE} & \scalea{MAE} & \scalea{MSE} & \scalea{MAE} & \scalea{MSE} & \scalea{MAE} & \scalea{MSE} & \scalea{MAE} & \scalea{MSE} & \scalea{MAE} \\
    \midrule

\multicolumn{2}{l}{\scalea{ETTm1}}
& \scalea{\bst{0.372}} & \scalea{\bst{0.391}}& \scalea{\subbst{0.375}} & \scalea{0.393}& \scalea{0.375} & \scalea{0.393}& \scalea{0.378} & \scalea{0.393}& \scalea{0.376} & \scalea{\subbst{0.392}}& \scalea{0.379} & \scalea{0.392}& \scalea{0.378} & \scalea{0.395}& \scalea{0.398} & \scalea{0.412}& \scalea{0.379} & \scalea{0.396}& \scalea{0.394} & \scalea{0.403}& \scalea{0.378} & \scalea{0.394} \\
\midrule
\multicolumn{2}{l}{\scalea{ETTm2}}
& \scalea{\bst{0.265}} & \scalea{\bst{0.312}}& \scalea{\subbst{0.268}} & \scalea{0.315}& \scalea{0.269} & \scalea{0.315}& \scalea{0.271} & \scalea{0.314}& \scalea{0.271} & \scalea{\subbst{0.313}}& \scalea{0.273} & \scalea{0.314}& \scalea{0.274} & \scalea{0.319}& \scalea{0.289} & \scalea{0.332}& \scalea{0.273} & \scalea{0.317}& \scalea{0.291} & \scalea{0.330}& \scalea{0.271} & \scalea{0.317} \\
\midrule
\multicolumn{2}{l}{\scalea{ETTh1}}
& \scalea{\bst{0.424}} & \scalea{\bst{0.426}}& \scalea{0.434} & \scalea{0.429}& \scalea{0.434} & \scalea{0.428}& \scalea{\subbst{0.430}} & \scalea{0.429}& \scalea{0.434} & \scalea{\subbst{0.428}}& \scalea{0.436} & \scalea{0.429}& \scalea{0.437} & \scalea{0.431}& \scalea{0.452} & \scalea{0.444}& \scalea{0.441} & \scalea{0.434}& \scalea{0.461} & \scalea{0.445}& \scalea{0.436} & \scalea{0.430} \\
\midrule
\multicolumn{2}{l}{\scalea{ETTh2}}
& \scalea{\bst{0.361}} & \scalea{\bst{0.391}}& \scalea{0.367} & \scalea{0.396}& \scalea{\subbst{0.365}} & \scalea{0.395}& \scalea{0.367} & \scalea{\subbst{0.394}}& \scalea{0.368} & \scalea{0.394}& \scalea{0.371} & \scalea{0.394}& \scalea{0.368} & \scalea{0.397}& \scalea{0.386} & \scalea{0.409}& \scalea{0.371} & \scalea{0.398}& \scalea{0.393} & \scalea{0.409}& \scalea{0.372} & \scalea{0.398} \\
\midrule
\multicolumn{2}{l}{\scalea{ECL}}
& \scalea{\bst{0.163}} & \scalea{\subbst{0.258}}& \scalea{\subbst{0.164}} & \scalea{0.259}& \scalea{0.164} & \scalea{0.259}& \scalea{0.165} & \scalea{0.259}& \scalea{0.166} & \scalea{\bst{0.258}}& \scalea{0.167} & \scalea{0.259}& \scalea{0.165} & \scalea{0.260}& \scalea{0.861} & \scalea{0.755}& \scalea{0.165} & \scalea{0.260}& \scalea{0.165} & \scalea{0.261}& \scalea{0.165} & \scalea{0.260} \\
\midrule
\multicolumn{2}{l}{\scalea{Weather}}
& \scalea{\bst{0.239}} & \scalea{\bst{0.265}}& \scalea{0.241} & \scalea{0.267}& \scalea{0.241} & \scalea{0.268}& \scalea{0.240} & \scalea{0.266}& \scalea{\subbst{0.240}} & \scalea{\subbst{0.265}}& \scalea{0.245} & \scalea{0.266}& \scalea{0.246} & \scalea{0.272}& \scalea{0.248} & \scalea{0.276}& \scalea{0.246} & \scalea{0.273}& \scalea{0.262} & \scalea{0.281}& \scalea{0.241} & \scalea{0.267} \\
\midrule
\multicolumn{2}{l}{\scalea{M5}}
& \scalea{\bst{0.152}} & \scalea{\bst{0.309}}& \scalea{0.156} & \scalea{\subbst{0.313}}& \scalea{\subbst{0.156}} & \scalea{0.313}& \scalea{0.161} & \scalea{0.319}& \scalea{0.168} & \scalea{0.327}& \scalea{0.162} & \scalea{0.321}& \scalea{0.158} & \scalea{0.317}& \scalea{0.160} & \scalea{0.318}& \scalea{0.159} & \scalea{0.317}& \scalea{0.162} & \scalea{0.322}& \scalea{0.168} & \scalea{0.327} \\
    \bottomrule
\end{tabular}
\begin{tablenotes}
    \item  \scriptsize \textit{Note}:  We fix the input length as 96 following~\citep{itransformer}. \bst{Bold} and \subbst{underlined} denote best and second-best results, respectively. \emph{Avg} indicates average results over forecast horizons: T=96, 192, 336 and 720.
\end{tablenotes}
\end{threeparttable}

\vspace{-1.2em}
\end{table}

\subsection{Overall performance}\label{sec:overall}

Table~\ref{tab:loss_avg} presents a comprehensive comparison between KMB-DF and competitive learning objectives. We have the following observations:
\ding{182} \textbf{The naive MSE objective exhibits suboptimal performance.} For example, it yields MSE/MAE values of 0.378/0.394 on ETTm1, 0.165/0.260 on ECL, and 0.168/0.327 on M5, all of which are higher than those of KMB-DF. This result is consistent with MSE optimizing each forecast step independently, without explicitly accounting for the autocorrelation among future labels described in Section~\ref{sec:obj}.
\ding{183} \textbf{Existing alternative objectives offer limited improvements over MSE.} QDF and DistDF reduce the MSE from 0.378 to 0.375 on ETTm1 and from 0.372 to 0.365 on ETTh2, respectively, while FreDF matches the best MAE on ECL and Weather. In contrast, shape-alignment objectives can yield substantially larger errors on some datasets: GDTW reports an MSE of 0.861 and an MAE of 0.755 on ECL. By aligning only a few predefined transformed components or statistics, these objectives cannot consistently achieve distribution balance, which leads to suboptimal performance.
\ding{184} \textbf{KMB-DF achieves the best overall results among the compared objectives.} It obtains the lowest MSE and MAE on almost all datasets, while tying the best MAE on ECL with FreDF. These results demonstrate the effectiveness of adaptively selected kernelized balancing functions for reducing distributional discrepancies.



\begin{wrapfigure}{r}{0.5\linewidth}
\vspace{-1.5em}
\begin{center}
\includegraphics[width=0.48\linewidth]{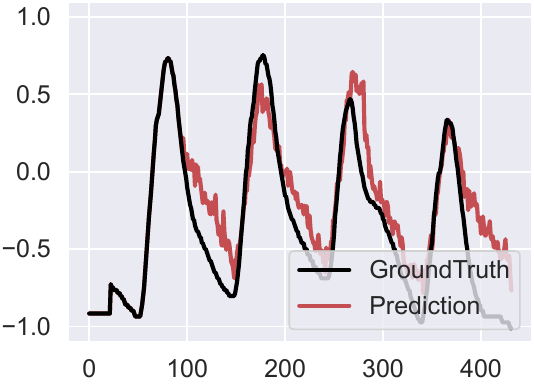}
\hfill
\includegraphics[width=0.48\linewidth]{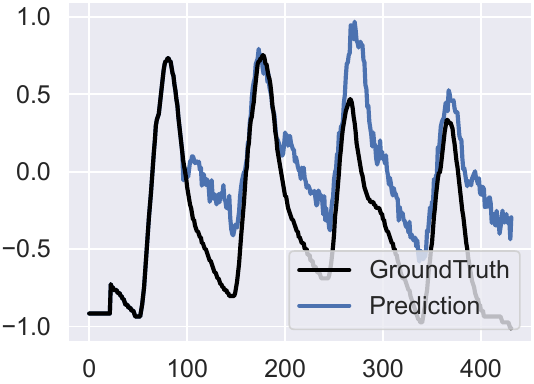}

\vspace{-7pt}
\caption{The forecast sequence of KMB-DF (in red) and MSE (in blue).}\label{fig:case}
\end{center}

\vspace{-9pt}
\end{wrapfigure}

\textbf{Showcases.}
We qualitatively compare the forecast sequences produced by models trained with KMB-DF and MSE. Figure~\ref{fig:case} presents a snapshot from ETTm2 under lookback window $\mathrm{H}=96$. Although the MSE-trained model captures overall periodic pattern, its prediction deviates noticeably around the transition between steps 200 and 300 and overestimates the subsequent peak. In contrast, KMB-DF follows these local variations more closely. By adaptively selecting kernelized balancing functions to align the joint distributions of labels and forecasts, KMB-DF enables the model to better capture the temporal dependencies within the label sequence.

\subsection{Ablation studies}\label{sec:ablation}
In this section, we dissect the contributions of the soft-margin relaxation and the selective balancing mechanism inherent in KMB-DF. \autoref{tab:system_ablation_app} examines the performance of variants by selectively enabling these components. The main findings are as follows:
\ding{182} \textbf{Soft-margin relaxation improves forecasting performance.} KMB-DF$^\dagger$ adds slack variables to DF while retaining all balancing functions. It reduces the average MSE from 0.378 to 0.374 on ETTm1 and from 0.436 to 0.430 on ETTh1, confirming the benefit of allowing bounded moment mismatches in the balancing objective.
\ding{183} \textbf{Selective balancing also benefits forecasting performance.} KMB-DF$^\ddagger$ uses the informativeness score in Eq.~(\ref{eq:delta}) while enforcing strict constraints. It reduces the average MSE by 0.79\% on ETTm1 and 2.06\% on ETTh1, demonstrating the effectiveness of selecting balancing functions rather than balancing all candidate functions uniformly.
\ding{184} \textbf{The two ingredients yield complementary gains.} This is supported by the leading performance of the full KMB-DF across all datasets, forecasting horizons, and metrics. For example, it reduces the average MSE by 1.59\% on ETTm1 and 2.75\% on ETTh1 relative to DF. This demonstrates their complementary benefits for forecasting.

\begin{table}[t]
\caption{Ablation study results.}\label{tab:system_ablation_app}
\setlength{\tabcolsep}{3.2pt}
\renewcommand{\arraystretch}{0.8}
\scriptsize
\centering
\begin{threeparttable}
\begin{tabular}{lcclccccccccccccccc}
    \toprule
    \multirow{2}{*}{Model} & \multirow{2}{*}{Soft-Margin} & \multirow{2}{*}{Selective} &\multirow{2}{*}{Data} && \multicolumn{2}{c}{T=96} && \multicolumn{2}{c}{T=192} && \multicolumn{2}{c}{T=336} && \multicolumn{2}{c}{T=720} && \multicolumn{2}{c}{Avg} \\
    \cmidrule{6-7} \cmidrule{9-10} \cmidrule{12-13} \cmidrule{15-16} \cmidrule{18-19}
    &&&&& MSE  & MAE && MSE & MAE && MSE & MAE && MSE & MAE && MSE & MAE \\

\midrule
\multirow{4}{*}{DF} & \multirow{4}{*}{\XSolidBrush} & \multirow{4}{*}{\XSolidBrush}
&   ETTm1 && 0.321 & 0.358 && 0.357 & 0.382 && 0.387 & 0.401 && 0.446 & 0.435 && 0.378 & 0.394 \\
&&& ETTh1 && 0.373 & 0.391 && 0.427 & 0.421 && 0.466 & 0.441 && 0.477 & 0.468 && 0.436 & 0.430 \\
&&& ECL && \subbst{0.137} & 0.232 && 0.153 & 0.247 && 0.168 & 0.266 && \subbst{0.199} & \subbst{0.294} && \subbst{0.165} & 0.260 \\
&&& Weather && 0.156 & 0.201 && 0.205 & 0.245 && 0.261 & 0.286 && 0.343 & 0.339 && 0.241 & 0.267 \\
\midrule
\multirow{4}{*}{KMB-DF$^\dagger$} & \multirow{4}{*}{\Checkmark} & \multirow{4}{*}{\XSolidBrush}
&   ETTm1 && \subbst{0.318} & \subbst{0.356} && \subbst{0.355} & 0.381 && \subbst{0.383} & \subbst{0.400} && \subbst{0.442} & \subbst{0.433} && \subbst{0.374} & \subbst{0.392} \\
&&& ETTh1 && \subbst{0.372} & 0.391 && \subbst{0.423} & 0.421 && 0.462 & 0.443 && 0.462 & 0.469 && 0.430 & 0.431 \\
&&& ECL && 0.137 & \subbst{0.232} && \subbst{0.153} & \subbst{0.247} && \subbst{0.168} & \subbst{0.265} && 0.201 & 0.294 && 0.165 & \subbst{0.259} \\
&&& Weather && 0.154 & 0.199 && \subbst{0.204} & \subbst{0.242} && 0.261 & 0.286 && 0.342 & 0.341 && 0.240 & 0.267 \\
\midrule
\multirow{4}{*}{KMB-DF$^\ddagger$} & \multirow{4}{*}{\XSolidBrush} & \multirow{4}{*}{\Checkmark}
&   ETTm1 && 0.318 & 0.356 && 0.356 & \subbst{0.380} && 0.384 & 0.401 && 0.442 & \subbst{0.433} && 0.375 & 0.393 \\
&&& ETTh1 && 0.372 & \bst{0.389} && 0.425 & \bst{0.420} && \bst{0.458} & \subbst{0.439} && \subbst{0.455} & \subbst{0.457} && \subbst{0.427} & \subbst{0.426} \\
&&& ECL && 0.137 & 0.232 && 0.153 & 0.247 && 0.168 & 0.266 && 0.200 & 0.294 && 0.165 & 0.260 \\
&&& Weather && \subbst{0.153} & \subbst{0.198} && 0.205 & 0.244 && \subbst{0.261} & \subbst{0.285} && \subbst{0.341} & \subbst{0.338} && \subbst{0.240} & \subbst{0.266} \\
\midrule
\multirow{4}{*}{KMB-DF} & \multirow{4}{*}{\Checkmark} & \multirow{4}{*}{\Checkmark}
&   ETTm1 && \bst{0.313} & \bst{0.354} && \bst{0.351} & \bst{0.379} && \bst{0.381} & \bst{0.400} && \bst{0.442} & \bst{0.433} && \bst{0.372} & \bst{0.391} \\
&&& ETTh1 && \bst{0.367} & \subbst{0.390} && \bst{0.422} & \subbst{0.420} && \subbst{0.460} & \bst{0.438} && \bst{0.447} & \bst{0.455} && \bst{0.424} & \bst{0.426} \\
&&& ECL && \bst{0.136} & \bst{0.231} && \bst{0.152} & \bst{0.245} && \bst{0.168} & \bst{0.265} && \bst{0.197} & \bst{0.291} && \bst{0.163} & \bst{0.258} \\
&&& Weather && \bst{0.152} & \bst{0.196} && \bst{0.204} & \bst{0.242} && \bst{0.260} & \bst{0.284} && \bst{0.339} & \bst{0.337} && \bst{0.239} & \bst{0.265} \\
    \bottomrule
\end{tabular}
\begin{tablenotes}
    \item \scriptsize \textit{Note}:  \bst{Bold} and \subbst{underlined} denote best and second-best results, respectively.
\end{tablenotes}
\end{threeparttable}

\vspace{-1em}
\end{table}

\subsection{Kernel function selection studies}
\label{sec:kernel_studies}

\begin{wraptable}{r}{0.5\linewidth}
\vspace{-2em}
\centering
\caption{Varying kernel function results.}
\setlength{\tabcolsep}{2pt}
\renewcommand{\arraystretch}{1}
\tiny
\label{tab:varying-kernel}
\begin{tabular}{l|cc|cc|cc|cc}
\hline

& \multicolumn{4}{c|}{ETTh1} & \multicolumn{4}{c}{ETTm1}\\
\hline
Kernel & MSE & $\Delta$MSE & MAE & $\Delta$MAE & MSE & $\Delta$MSE & MAE & $\Delta$MAE \\
\hline
DF & 0.436 & - & 0.430 & - & 0.378 & - & 0.394 & - \\
Linear & 0.435 & \reduce{0.3\%$\downarrow$}& 0.428 & \reduce{0.4\%$\downarrow$} & 0.375 & \reduce{0.7\%$\downarrow$}& 0.393 & \reduce{0.3\%$\downarrow$} \\
Polynomial & 0.434 & \reduce{0.4\%$\downarrow$}& 0.430 & \reduce{0.1\%$\downarrow$} & 0.375 & \reduce{0.7\%$\downarrow$}& 0.393 & \reduce{0.3\%$\downarrow$} \\
Sigmoid & 0.433 & \reduce{0.6\%$\downarrow$}& 0.430 & \reduce{0.0\%$\uparrow$} & 0.375 & \reduce{0.9\%$\downarrow$}& 0.392 & \reduce{0.5\%$\downarrow$} \\
Gaussian & 0.434 & \reduce{0.5\%$\downarrow$}& 0.430 & \reduce{0.1\%$\downarrow$} & 0.374 & \reduce{1.0\%$\downarrow$}& 0.392 & \reduce{0.4\%$\downarrow$} \\
Exponential & 0.424 & \reduce{2.7\%$\downarrow$}& 0.426 & \reduce{1.0\%$\downarrow$} & 0.372 & \reduce{1.6\%$\downarrow$}& 0.391 & \reduce{0.7\%$\downarrow$} \\
\hline
\end{tabular}

\vspace{-5pt}
\end{wraptable}


In this section, we evaluate the impact of the kernel function choice $\mathcal{K}(\cdot, \cdot)$ in instantiating the balancing constraints. We compare the exponential kernel against Linear, Polynomial, Sigmoid, and Gaussian kernels, as detailed in Table~\ref{tab:varying-kernel}. We make the following observations:
\ding{182} \textbf{All tested kernels reduce MSE relative to DF on both datasets.} For example, the linear kernel reduces the MSE from 0.378 to 0.375 on ETTm1. This result suggests that kernelized moment-balancing constraints complement the point-wise MSE objective by accounting for dependencies among future labels induced by label autocorrelation.
\ding{183} \textbf{The exponential kernel yields the lowest errors among the tested kernels.} It reduces the MSE by 2.7\% on ETTh1 relative to DF, and also gives the lowest MAE on both datasets. These results demonstrate that the exponential kernel is an effective choice for adaptively balancing nonlinear distributional discrepancies between labels and forecasts.



\begin{figure}[b]
\vspace{-1em}
\begin{center}
\subfigure[ECL with MSE]{\includegraphics[width=0.24\linewidth]{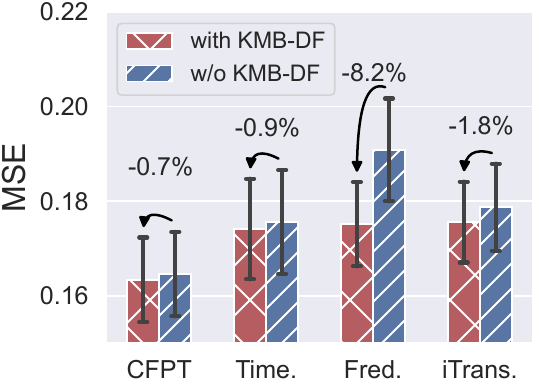}}
\hfill
\subfigure[ECL with MAE]{\includegraphics[width=0.24\linewidth]{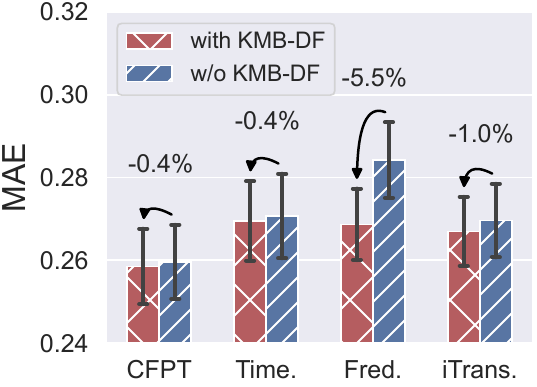}}
\hfill
\subfigure[Weather with MSE]{\includegraphics[width=0.24\linewidth]{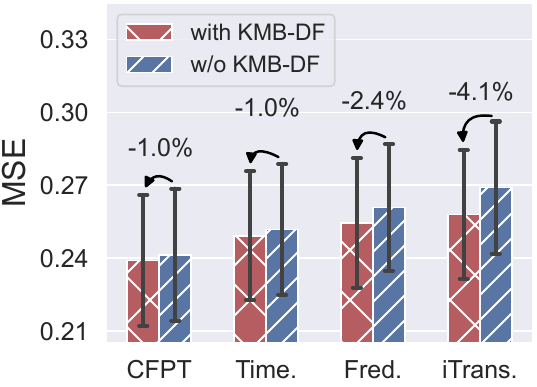}}
\hfill
\subfigure[Weather with MAE]{\includegraphics[width=0.24\linewidth]{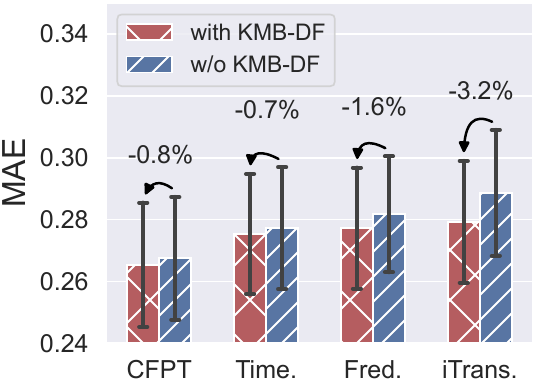}}

\vspace{-5pt}
\caption{Improvement of KMB-DF applied to different forecast models, shown with colored bars for means over forecast lengths (96, 192, 336, 720) and error bars for 50\% confidence intervals. }
\label{fig:backbone}
\end{center}

\end{figure}

\subsection{Generalization Studies}
\label{sec:generalize}


In this section, we assess the generalizability of KMB-DF by integrating it into four representative state-of-the-art forecast models: CFPT~\citep{CFPT}, TimeBridge~\citep{TimeBridge}, Fredformer~\citep{Fredformer}, and iTransformer~\citep{itransformer}. The results are presented in Figure~\ref{fig:backbone} with the following observations.
\ding{182} \textbf{KMB-DF improves the performance of backbones developed in different periods.} As illustrated by the ECL results, KMB-DF improves both earlier backbones, including Fredformer and iTransformer, and more recent backbones, including TimeBridge and CFPT.
\ding{183} \textbf{KMB-DF improves the performance of backbones with different architectural complexity.} As illustrated by the Weather results, it improves both the linear CFPT model and the more expressive Transformer-based models.
These results demonstrate that KMB-DF is a general backbone-agnostic learning objective rather than an enhancement specialized for a particular backbone or dataset. By adaptively selecting kernelized balancing functions, it provides an objective-level complement to point-wise losses that overlook label autocorrelation.

\subsection{Hyperparameter sensitivity}\label{sec:hyper}
In this section, we examine the sensitivity of KMB-DF to the penalty strength $\alpha$, the tolerated imbalance margin $\mathrm{C}$, and the number of selected balancing functions $\mathrm{K}$. The results are presented in Tables~\ref{tab:varying-alpha}, \ref{tab:varying-C}, and \ref{tab:varying-K}, respectively, with the following observations.

\begin{wraptable}{r}{0.5\linewidth}
\vspace{-2em}
\centering
\caption{Varying $\alpha$ results.}
\label{tab:varying-alpha}
\scriptsize
\setlength{\tabcolsep}{2.1pt}
\renewcommand{\arraystretch}{0.9}
\begin{tabular}{l|cc|cc|cc|cc}
\hline
& \multicolumn{4}{c|}{ETTh1} & \multicolumn{4}{c}{ETTm1}\\
\hline
$\alpha$ & MSE & $\Delta$MSE & MAE & $\Delta$MAE & MSE & $\Delta$MSE & MAE & $\Delta$MAE \\
\hline
DF & 0.436 & - & 0.430 & - & 0.378 & - & 0.394 & - \\
0.1 & 0.431 & \reduce{1.2\%$\downarrow$}& 0.429 & \reduce{0.4\%$\downarrow$} & 0.374 & \reduce{1.0\%$\downarrow$}& 0.392 & \reduce{0.4\%$\downarrow$} \\
0.3 & 0.431 & \reduce{1.2\%$\downarrow$}& 0.429 & \reduce{0.4\%$\downarrow$} & 0.373 & \reduce{1.2\%$\downarrow$}& 0.393 & \reduce{0.2\%$\downarrow$} \\
0.5 & 0.433 & \reduce{0.7\%$\downarrow$}& 0.430 & \reduce{0.1\%$\downarrow$} & 0.374 & \reduce{1.1\%$\downarrow$}& 0.392 & \reduce{0.5\%$\downarrow$} \\
0.7 & 0.426 & \reduce{2.3\%$\downarrow$}& 0.426 & \reduce{0.9\%$\downarrow$} & 0.374 & \reduce{1.1\%$\downarrow$}& 0.392 & \reduce{0.5\%$\downarrow$} \\
0.9 & 0.431 & \reduce{1.1\%$\downarrow$}& 0.429 & \reduce{0.3\%$\downarrow$} & 0.373 & \reduce{1.3\%$\downarrow$}& 0.392 & \reduce{0.6\%$\downarrow$} \\
\hline
\end{tabular}

\vspace{-1em}
\end{wraptable}

\textbf{Penalty strength $\alpha$.}
\ding{182} \textbf{All nonzero values of $\alpha$ improve performance over DF.} On both datasets, every tested value achieves lower MSE than DF. For example, $\alpha=0.1$ reduces the MSE from 0.436 to 0.431 on ETTh1. These results demonstrate the benefit of the kernelized moment-balancing penalty across a broad range of weights.
\ding{183} \textbf{Intermediate values of $\alpha$ achieve the lowest errors.} In particular, $\alpha=0.7$ gives the lowest MSE and MAE on ETTh1, while $\alpha=0.9$ attains best performance on ETTm1. This suggests that the balancing penalty provides complementary information to the point-wise MSE objective.

\begin{wraptable}{r}{0.5\linewidth}
\vspace{-2em}
\centering
\caption{Varying $\mathrm{C}$ results.}
\label{tab:varying-C}
\scriptsize
\setlength{\tabcolsep}{1.5pt}
\renewcommand{\arraystretch}{0.9}
\begin{tabular}{l|cc|cc|cc|cc}
\hline
& \multicolumn{4}{c|}{ETTh1} & \multicolumn{4}{c}{ETTm1}\\
\hline
$\mathrm{C}$ & MSE & $\Delta$MSE & MAE & $\Delta$MAE & MSE & $\Delta$MSE & MAE & $\Delta$MAE \\
\hline
DF & 0.436 & - & 0.430 & - & 0.378 & - & 0.394 & - \\
0.0005 & 0.430 & \reduce{1.4\%$\downarrow$}& 0.429 & \reduce{0.3\%$\downarrow$} & 0.374 & \reduce{1.0\%$\downarrow$}& 0.392 & \reduce{0.4\%$\downarrow$} \\
0.001 & 0.424 & \reduce{2.6\%$\downarrow$}& 0.426 & \reduce{0.9\%$\downarrow$} & 0.374 & \reduce{1.0\%$\downarrow$}& 0.392 & \reduce{0.4\%$\downarrow$} \\
0.005 & 0.429 & \reduce{1.5\%$\downarrow$}& 0.429 & \reduce{0.3\%$\downarrow$} & 0.373 & \reduce{1.3\%$\downarrow$}& 0.392 & \reduce{0.6\%$\downarrow$} \\
0.01 & 0.430 & \reduce{1.4\%$\downarrow$}& 0.428 & \reduce{0.5\%$\downarrow$} & 0.373 & \reduce{1.3\%$\downarrow$}& 0.392 & \reduce{0.6\%$\downarrow$} \\
0.05 & 0.434 & \reduce{0.5\%$\downarrow$}& 0.430 & \reduce{0.1\%$\downarrow$} & 0.374 & \reduce{1.0\%$\downarrow$}& 0.392 & \reduce{0.4\%$\downarrow$} \\
\hline
\end{tabular}

\vspace{-1em}
\end{wraptable}

\textbf{Tolerated imbalance margin $\mathrm{C}$.}
\ding{182} \textbf{All tested values of $\mathrm{C}$ improve performance over DF.} Each margin reduces MSE on both datasets. For example, $\mathrm{C}=0.0005$ reduces the MSE from 0.436 to 0.430, demonstrating that soft-margin relaxation benefits forecasting performance across a broad range of margins.
\ding{183} \textbf{Smaller values of $\mathrm{C}$ generally achieve better performance.} For example, $\mathrm{C}=0.001$ achieves the largest MSE reduction of 2.6\%. In contrast, increasing $\mathrm{C}$ to 0.05 yields only a 0.5\% reduction. Thus, relatively small tolerated mismatches provide effective default choices.

\begin{wraptable}{r}{0.5\linewidth}
\vspace{-2em}
\centering
\caption{Varying $\mathrm{K}$ results.}
\setlength{\tabcolsep}{2.1pt}
\renewcommand{\arraystretch}{0.9}
\scriptsize
\label{tab:varying-K}
\begin{tabular}{l|cc|cc|cc|cc}
\hline

& \multicolumn{4}{c|}{ETTh1} & \multicolumn{4}{c}{ETTm1}\\
\hline
$\mathrm{K}$ & MSE & $\Delta$MSE & MAE & $\Delta$MAE & MSE & $\Delta$MSE & MAE & $\Delta$MAE \\
\hline
DF & 0.436 & - & 0.430 & - & 0.378 & - & 0.394 & - \\
1 & 0.432 & \reduce{0.8\%$\downarrow$}& 0.429 & \reduce{0.3\%$\downarrow$} & 0.374 & \reduce{1.1\%$\downarrow$}& 0.392 & \reduce{0.5\%$\downarrow$} \\
2 & 0.431 & \reduce{1.1\%$\downarrow$}& 0.429 & \reduce{0.4\%$\downarrow$} & 0.374 & \reduce{1.1\%$\downarrow$}& 0.392 & \reduce{0.5\%$\downarrow$} \\
3 & 0.427 & \reduce{2.1\%$\downarrow$}& 0.426 & \reduce{0.9\%$\downarrow$} & 0.374 & \reduce{1.1\%$\downarrow$}& 0.392 & \reduce{0.4\%$\downarrow$} \\
4 & 0.432 & \reduce{0.9\%$\downarrow$}& 0.429 & \reduce{0.2\%$\downarrow$} & 0.374 & \reduce{0.9\%$\downarrow$}& 0.393 & \reduce{0.3\%$\downarrow$} \\
5 & 0.426 & \reduce{2.4\%$\downarrow$}& 0.427 & \reduce{0.9\%$\downarrow$} & 0.374 & \reduce{1.1\%$\downarrow$}& 0.392 & \reduce{0.4\%$\downarrow$} \\
\hline
\end{tabular}

\vspace{-8pt}
\end{wraptable}

\textbf{Number of balancing functions $\mathrm{K}$.}
\ding{182} \textbf{Using one adaptively selected balancing function improves performance over DF.} With $\mathrm{K}=1$, the MSE decreases from 0.436 to 0.432 on ETTh1 and from 0.378 to 0.374 on ETTm1. This demonstrates that even a single function selected by the proposed adaptive search provides a useful complement to the point-wise MSE objective.
\ding{183} \textbf{A small number of adaptively selected functions achieves near-optimal performance.} On ETTh1, $\mathrm{K}=3$ achieves the lowest MAE, while $\mathrm{K}=5$ gives the lowest MSE; on ETTm1, all tested values yield similar errors. Thus, $\mathrm{K}=3$ is an effective default that maintains a small selection size without extensive tuning.

\section{Conclusion}


In this work, we identify a fundamental limitation in existing deep time-series forecasting objectives: they fail to satisfy Imbens' criterion for true distribution balance, as they enforce moment matching only on a limited set of predefined functions. To address this, we propose Direct Forecasting with Kernelized Moment Balancing (KMB-DF) to ensure sufficient distribution alignment. Instead of relying on finite statistics, KMB-DF leverages reproducing kernel Hilbert spaces (RKHS) to adaptively select infinite-dimensional balancing functions. We theoretically justify balancing joint distributions as a tractable proxy for aligning conditional distributions. Experiments demonstrate that KMB-DF consistently achieves state-of-the-art performance by effectively mitigating autocorrelation bias.

\textbf{Limitations.} 
While KMB-DF enforces global alignment through kernelized moment matching, it does not dictate sample-wise temporal correspondence. Consequently, KMB-DF serves as a regularizer rather than a standalone objective. It requires a point-wise loss (\eg MSE) to anchor element-wise alignment, correcting distributional discrepancies while maintaining temporal precision.

%
%

\subsection*{AI use statement}

We used generative AI tools for language polishing, formatting and coding assistance, as well as preliminary brainstorming and literature assistance. Generative AI tools were not used to formulate the core research method, conduct experiments, perform the empirical or theoretical analyses, or determine the scientific conclusions. All AI-assisted text, ideas, and code were critically reviewed and verified by the authors; the core methodology, experiments, analyses, and final conclusions were developed and validated by the authors. We take full responsibility for the final content of this work, including all text, claims, and artifacts produced with the aid of generative AI.

\bibliographystyle{iclr2027_conference}
\bibliography{bib/abbr,bib/timeseries,bib/main,bib/submain,bib/supp,bib/ot,bib/causality}

@STRING{IEEE_J_KDE        = "{IEEE} Trans. Knowl. Data Eng."}

@STRING{IEEE_J_PAMI       = "{IEEE} Trans. Pattern Anal. Mach. Intell."}

@STRING{IEEE_J_GRS        = "{IEEE} Trans. Geosci. Remote Sens."}

@STRING{IEEE_J_CE         = "{IEEE} Trans. Consum. Electron."}

@STRING{J_TMLR            = "Trans. Mach. Learn. Res."}

@STRING{J_JMLR            = "J. Mach. Learn. Res."}

@STRING{ELS_IS            = "Inf. Sci"}

@STRING{P_NIPS         = "Proc. Adv. Neural Inf. Process. Syst."}

@STRING{P_ICML         = "Proc. Int. Conf. Mach. Learn."}

@STRING{P_ICLR         = "Proc. Int. Conf. Learn. Represent."}

@STRING{P_ICASSP         = "Proc. IEEE Int. Conf. Acoust. Speech Signal Process."}

@STRING{P_AAAI         = "Proc. AAAI Conf. Artif. Intell."}

@STRING{P_SIGKDD         = "Proc. ACM SIGKDD Int. Conf. Knowl. Discovery Data Mining"}

@STRING{P_CIKM         = "Proc. ACM Int. Conf. Inf. Knowl. Manag."}

@STRING{P_VLDB         = "Proc. {VLDB} Endow."}

@STRING{P_EMNLP         = "Proc. Conf. Empirical Methods Natural Lang. Process."}

@article{cbps,
  title={Covariate balancing propensity score},
  author={Imai, Kosuke and Ratkovic, Marc},
  journal={J. R. Stat. Soc. Series. B. Stat. Methodol.},
  volume={76},
  number={1},
  pages={243--263},
  year={2014},
}

@article{cbgps,
  title={Covariate balancing propensity score for a continuous treatment: application to the efficacy of political advertisements},
  author={Fong, Christian and Hazlett, Chad and Imai, Kosuke},
  journal={Ann. Appl. Stat.},
  volume={12},
  number={1},
  pages={156--177},
  year={2018},
}

@article{imbens2015causal,
  title={Causal inference in statistics, social, and biomedical sciences},
  author={Imbens, Guido W and Rubin, Donald B},
  journal={New York},
  volume={517},
  year={2015}
}

@inproceedings{chen2023mode,
  author    = {Zhichao Chen and
               Leilei Ding and
               Zhixuan Chu and
               Yucheng Qi and
               Jianmin Huang and
               Hao Wang},
  title     = {Monotonic neural ordinary differential equation: time-series forecasting
               for cumulative data},
  booktitle = P_CIKM,
  pages     = {4523--4529},
  year      = {2023},
}

@inproceedings{wang2025iclrfredf,
  author    = {Hao Wang and
               Licheng Pan and
               Yuan Shen and
               Zhichao Chen and
               Degui Yang and
               Yifei Yang and
               Sen Zhang and
               Xinggao Liu and
               Haoxuan Li and
               Dacheng Tao},
  title     = {FreDF: learning to forecast in the frequency domain},
  booktitle = P_ICLR,
  pages     = {1--9},
  year      = {2025},
}

@article{wang2025nipstimeo1,
  title={Time-o1: time-series forecasting needs transformed label alignment},
  author={Wang, Hao and Pan, Licheng and Chen, Zhichao and Chen, Xu and Dai, Qingyang and Wang, Lei and Li, Haoxuan and Lin, Zhouchen},
  journal=P_NIPS,
  year={2025},
  pages = {1--9}
}

@inproceedings{wang2026iclrqdf,
  title={Quadratic direct forecast for training multi-step time-series forecast models},
  author={Wang, Hao and Pan, Licheng and Lu, Yuan and Chen, Zhichao and Liu, Tianqiao and He, Shuting and Chu, Zhixuan and Wen, Qingsong and Li, Haoxuan and Lin, Zhouchen},
  booktitle = P_ICLR,
  pages     = {1--9},
  year      = {2026},
}

@inproceedings{wang2026iclrdistdf,
  title={DistDF: time-series forecasting needs joint-distribution Wasserstein alignment},
  author={Wang, Hao and Pan, Licheng and Lu, Yuan and Chu, Zhixuan and Li, Xiaoxi and He, Shuting and Chen, Zhichao and Li, Haoxuan and Wen, Qingsong and Lin, Zhouchen},
  booktitle = P_ICLR,
  pages     = {1--9},
  year      = {2026},
}

@inproceedings{adam,
  author    = {Diederik P. Kingma and
               Jimmy Ba},
  title     = {Adam: a method for stochastic optimization},
  booktitle = P_ICLR,
  pages     = {1--9},
  year      = {2015},
}

@article{wang2026museagent,
  title={MuSEAgent: a multimodal reasoning agent with stateful experiences},
  author={Wang, Shijian and Jin, Jiarui and Fu, Runhao and Yan, Zexuan and Wang, Xingjian and Hu, Mengkang and Wang, Eric and Li, Xiaoxi and Zhang, Kangning and Yao, Li and others},
  journal={arXiv preprint arXiv:2603.27813},
  year={2026}
}

@inproceedings{li2023propensity,
  author       = {Haoxuan Li and
                  Yanghao Xiao and
                  Chunyuan Zheng and
                  Peng Wu and
                  Peng Cui},
  title        = {Propensity matters: measuring and enhancing balancing for recommendation},
  booktitle    = P_ICML,
  volume       = {202},
  pages        = {20182--20194},
  year         = {2023}
}

@inproceedings{qiutfb,
title   = {TFB: towards comprehensive and fair benchmarking of time series forecasting methods},
author  = {Xiangfei Qiu and Jilin Hu and Lekui Zhou and Xingjian Wu and Junyang Du and Buang Zhang and Chenjuan Guo and Aoying Zhou and Christian S. Jensen and Zhenli Sheng and Bin Yang},
booktitle = P_VLDB,
pages   = {2363--2377},
year    = {2024}
}

@article{qiudbloss,
  title={DBLoss: decomposition-based loss function for time series forecasting},
  author={Qiu, Xiangfei and Wu, Xingjian and Cheng, Hanyin and Liu, Xvyuan and Guo, Chenjuan and Hu, Jilin and Yang, Bin},
  journal=P_NIPS,
  year={2025}
}

@inproceedings{mars,
  title={Mars: a financial market simulation engine powered by generative foundation model},
  author={Li, Junjie and Liu, Yang and Liu, Weiqing and Fang, Shikai and Wang, Lewen and Xu, Chang and Bian, Jiang},
  booktitle=P_ICLR,
  volume={2025},
  pages={39490--39524},
  year={2025}
}

@article{kmmd,
  title={A kernel two-sample test},
  author={Gretton, Arthur and Borgwardt, Karsten M and Rasch, Malte J and Sch{\"o}lkopf, Bernhard and Smola, Alexander},
  journal=J_JMLR,
  volume={13},
  number={1},
  pages={723--773},
  year={2012},
}

@article{M5,
  title={M5 accuracy competition: Results, findings, and conclusions},
  author={Makridakis, Spyros and Spiliotis, Evangelos and Assimakopoulos, Vassilios},
  journal={Int. J. Forecast.},
  volume={38},
  number={4},
  pages={1346--1364},
  year={2022},
  publisher={Elsevier}
}

@book{mohri2018foundations,
  title={Foundations of machine learning},
  author={Mohri, Mehryar and Rostamizadeh, Afshin and Talwalkar, Ameet},
  year={2018},
  publisher={MIT press}
}

@article{liu2026memgui,
  title={MemGUI-Agent: An End-to-End Long-Horizon Mobile GUI Agent with Proactive Context Management},
  author={Liu, Guangyi and Wu, Gao and Liu, Congxiao and Zhao, Pengxiang and Liu, Liang and Li, Mading and Zhang, Qi and Wang, Mengyan and Guo, Liang and Liu, Yong},
  journal={arXiv preprint arXiv:2606.19926},
  year={2026}
}

@article{xu2026draftingagentsjusttokens,
    title={Drafting Agents, Not Just Tokens: Speculative Recursive Self-Improvement for Agentic Systems}, 
    author={Xin Xu and Ke Cheng and Hao Wang and Yangru Huang and Yanghao Xiao and Zhichao Chen and Zhisheng Zhong and Xun Liu and Yi Li and Qian Li and Haibin Lv and Jun Zhang and Jie Jiang and Zhouchen Lin},
    year={2026},
    journal={arXiv preprint arXiv:2610.16491}
}

@article{LSTM,
  title={Long short-term memory},
  author={Hochreiter, Sepp and Schmidhuber, J{\"u}rgen},
  journal={Neural Comput.},
  volume={9},
  number={8},
  pages={1735--1780},
  year={1997},
}

@inproceedings{GRU,
  title={Learning phrase representations using RNN encoder--decoder for statistical machine translation},
  author={Cho, Kyunghyun and Van Merri{\"e}nboer, Bart and Gul{\c{c}}ehre, {\c{C}}a{\u{g}}lar and Bahdanau, Dzmitry and Bougares, Fethi and Schwenk, Holger and Bengio, Yoshua},
  booktitle=P_EMNLP,
  pages={1724--1734},
  year={2014}
}

@article{DeepAR,
  title={DeepAR: Probabilistic forecasting with autoregressive recurrent networks},
  author={Salinas, David and Flunkert, Valentin and Gasthaus, Jan and Januschowski, Tim},
  journal={Int. J. Forecast},
  volume={36},
  number={3},
  pages={1181--1191},
  year={2020},
}

@inproceedings{P-sLSTM,
  title={Unlocking the power of lstm for long term time series forecasting},
  author={Kong, Yaxuan and Wang, Zepu and Nie, Yuqi and Zhou, Tian and Zohren, Stefan and Liang, Yuxuan and Sun, Peng and Wen, Qingsong},
  booktitle=P_AAAI,
  volume={39},
  number={11},
  pages={11968--11976},
  year={2025}
}

@inproceedings{gu2023mamba,
  title={Mamba: Linear-time sequence modeling with selective state spaces},
  author={Gu, Albert and Dao, Tri},
  booktitle={Proc. Conf. Lang. Model.},
  year={2023}
}

@article{MixMamba,
  title={MixMamba: Time series modeling with adaptive expertise},
  author={Alkilane, Khaled and He, Yihang and Lee, Der-Horng},
  journal={Inf. Fusion},
  volume={112},
  pages={102589},
  year={2024},
}

@inproceedings{MICN,
  title={Micn: Multi-scale local and global context modeling for long-term series forecasting},
  author={Wang, Huiqiang and Peng, Jian and Huang, Feihu and Wang, Jince and Chen, Junhui and Xiao, Yifei},
  booktitle=P_ICLR,
  year={2023}
}

@inproceedings{Timesnet,
  title={TimesNet: Temporal 2D-Variation Modeling for General Time Series Analysis},
  author={Wu, Haixu and Hu, Tengge and Liu, Yong and Zhou, Hang and Wang, Jianmin and Long, Mingsheng},
  booktitle=P_ICLR,
  year={2023}
}

@inproceedings{Moderntcn,
  title={Moderntcn: A modern pure convolution structure for general time series analysis},
  author={Luo, Donghao and Wang, Xue},
  booktitle=P_ICLR,
  pages={1--43},
  year={2024}
}

@article{TiDE,
  title={Long-term Forecasting with TiDE: Time-series Dense Encoder},
  author={Das, Abhimanyu and Kong, Weihao and Leach, Andrew and Sen, Rajat and Yu, Rose},
  journal=J_TMLR,
  year={2023}
}

@inproceedings{DLinear,
  title={Are Transformers Effective for Time Series Forecasting?},
  author={Ailing Zeng and Muxi Chen and Lei Zhang and Qiang Xu},
  booktitle=P_AAAI,
  year={2023}
}

@inproceedings{FreTS,
  title={Frequency-domain MLPs are More Effective Learners in Time Series Forecasting},
  author={Yi, Kun and Zhang, Qi and Fan, Wei and Wang, Shoujin and Wang, Pengyang and He, Hui and An, Ning and Lian, Defu and Cao, Longbing and Niu, Zhendong},
  booktitle=P_NIPS,
  year={2023}
}

@article{Autotimes,
  title={Autotimes: Autoregressive time series forecasters via large language models},
  author={Liu, Yong and Qin, Guo and Huang, Xiangdong and Wang, Jianmin and Long, Mingsheng},
  journal=P_NIPS,
  volume={37},
  pages={122154--122184},
  year={2024}
}

@article{GPT4TS,
  title={One fits all: Power general time series analysis by pretrained lm},
  author={Zhou, Tian and Niu, Peisong and Sun, Liang and Jin, Rong and others},
  journal=P_NIPS,
  volume={36},
  pages={43322--43355},
  year={2023}
}

@inproceedings{Simpletm,
  title={SimpleTM: A Simple Baseline for Multivariate Time Series Forecasting},
  author={Chen, Hui and Luong, Viet and Mukherjee, Lopamudra and Singh, Vikas},
  booktitle=P_ICLR,
  year={2025}
}

@inproceedings{PatchTST,
  title={A Time Series is Worth 64 Words: Long-term Forecasting with Transformers},
  author={Nie, Yuqi and Nguyen, Nam H and Sinthong, Phanwadee and Kalagnanam, Jayant},
  booktitle=P_ICLR,
  year={2023}
}

@article{Time-LLM,
  title={Time-llm: Time series forecasting by reprogramming large language models},
  author={Jin, Ming and Wang, Shiyu and Ma, Lintao and Chu, Zhixuan and Zhang, James Y and Shi, Xiaoming and Chen, Pin-Yu and Liang, Yuxuan and Li, Yuan-Fang and Pan, Shirui and others},
  journal=P_ICLR,
  year={2024}
}

@inproceedings{Autoformer,
  title={Autoformer: Decomposition Transformers with {Auto-Correlation} for Long-Term Series Forecasting},
  author={Haixu Wu and Jiehui Xu and Jianmin Wang and Mingsheng Long},
  booktitle=P_NIPS,
  year={2021}
}

@inproceedings{Informer,
  title={Informer: Beyond efficient transformer for long sequence time-series forecasting},
  author={Li, Jianxin and Hui, Xiong and Zhang, Wancai},
  booktitle=P_AAAI,
  year={2021}
}

@inproceedings{itransformer,
  title={itransformer: Inverted transformers are effective for time series forecasting},
  author={Liu, Yong and Hu, Tengge and Zhang, Haoran and Wu, Haixu and Wang, Shiyu and Ma, Lintao and Long, Mingsheng},
  booktitle=P_ICLR,
  year={2024}
}

@inproceedings{fedformer,
  title={{FEDformer}: Frequency enhanced decomposed transformer for long-term series forecasting},
  author={Zhou, Tian and Ma, Ziqing and Wen, Qingsong and Wang, Xue and Sun, Liang and Jin, Rong},
  booktitle=P_ICML,
  year={2022}
}

@inproceedings{Fredformer,
  title={Fredformer: Frequency debiased transformer for time series forecasting},
  author={Piao, Xihao and Chen, Zheng and Murayama, Taichi and Matsubara, Yasuko and Sakurai, Yasushi},
  booktitle=P_SIGKDD,
  pages={2400--2410},
  year={2024}
}

@article{tsinghuasurvey,
  title={Deep time series models: A comprehensive survey and benchmark},
  author={Wang, Yuxuan and Wu, Haixu and Dong, Jiaxiang and Liu, Yong and Long, Mingsheng and Wang, Jianmin},
  journal={arXiv preprint arXiv:2407.13278},
  year={2024}
}

@inproceedings{tqnet,
  title={Temporal Query Network for Efficient Multivariate Time Series Forecasting},
  author={Lin, Shengsheng and Chen, Haojun and Wu, Haijie and Qiu, Chunyun and Lin, Weiwei},
  booktitle=P_ICML,
  year={2025}
}

@inproceedings{lin2024cyclenet,
  title={Cyclenet: Enhancing time series forecasting through modeling periodic patterns},
  author={Lin, Shengsheng and Lin, Weiwei and Hu, Xinyi and Wu, Wentai and Mo, Ruichao and Zhong, Haocheng},
  booktitle=P_NIPS,
  volume={37},
  pages={106315--106345},
  year={2024}
}

@article{application_weather2,
  title={Interpretable weather forecasting for worldwide stations with a unified deep model},
  author={Wu, Haixu and Zhou, Hang and Long, Mingsheng and Wang, Jianmin},
  journal={Nat. Mach. Intell.},
  pages={1--10},
  year={2023},
  publisher={Nature Publishing Group UK London}
}

@article{koopman,
  title={From Fourier to Koopman: Spectral methods for long-term time series prediction},
  author={Lange, Henning and Brunton, Steven L and Kutz, J Nathan},
  journal={Journal of Machine Learning Research},
  volume={22},
  number={41},
  pages={1--38},
  year={2021}
}

@article{OLinear,
  title={OLinear: A Linear Model for Time Series Forecasting in Orthogonally Transformed Domain},
  author={Yue, Wenzhen and Liu, Yong and Li, Haoxuan and Wang, Hao and Ying, Xianghua and Guo, Ruohao and Xing, Bowei and Shi, Ji},
  journal=P_NIPS,
  year={2025}
}

@inproceedings{soft-dtw,
  title={Soft-dtw: a differentiable loss function for time-series},
  author={Cuturi, Marco and Blondel, Mathieu},
  booktitle=P_ICML,
  pages={894--903},
  year={2017},
  organization={PMLR}
}

@inproceedings{psloss,
  title={Patch-wise Structural Loss for Time Series Forecasting},
  author={Kudrat, Dilfira and Xie, Zongxia and Sun, Yanru and Jia, Tianyu and Hu, Qinghua},
  booktitle=P_ICML,
  year={2025}
}

@inproceedings{GDTW,
  title={Gdtw: A novel differentiable dtw loss for time series tasks},
  author={Liu, Xiang and Li, Naiqi and Xia, Shu-Tao},
  booktitle=P_ICASSP,
  pages={2860--2864},
  year={2021},
  organization={IEEE}
}

@article{timesql,
  title={TimeSQL: Improving multivariate time series forecasting with multi-scale patching and smooth quadratic loss},
  author={Wang, Haoxin and Li, Bixiong and Fan, Songhai and Wu, Yuankai and Liu, Xianggen},
  journal=ELS_IS,
  volume={671},
  pages={120652},
  year={2024},
}

@article{softdtw_health2,
  title={Real-Time Bidirectional Haptic Data Forecasting in Teleoperated Telemedicine Systems},
  author={Liao, Zijun and Zheng, Li and Xie, Xinru and Wu, Hairui and Liu, Fangjia},
  journal={IEEE Transactions on Medical Robotics and Bionics},
  year={2025},
  
}

@ARTICLE{softdtw_geoscience,
  author={Song, Cao and Lu, Minghui and Wang, Yuqing and Lu, Wenkai},
  journal=IEEE_J_GRS, 
  title={A Dynamic Time Warping Loss-Based Closed-Loop CNN for Seismic Impedance Inversion}, 
  year={2022},
  volume={60},
  number={},
  pages={1-13},
}

@article{Dilate,
  title={Shape and time distortion loss for training deep time series forecasting models},
  author={Le Guen, Vincent and Thome, Nicolas},
  journal=P_NIPS,
  volume={32},
  year={2019}
}

@article{STRIPE2,
  author={Le Guen, Vincent and Thome, Nicolas},
  journal=IEEE_J_PAMI, 
  title={Deep Time Series Forecasting With Shape and Temporal Criteria}, 
  year={2023},
  volume={45},
  number={1},
  pages={342-355},
}

@article{tdalign,
  title={Modeling temporal dependencies within the target for long-term time series forecasting},
  author={Xiong, Qi and Tang, Kai and Ma, Minbo and Zhang, Ji and Xu, Jie and Li, Tianrui},
  journal=IEEE_J_KDE,
  year={2025},
}

@article{AST,
  title={Adversarial sparse transformer for time series forecasting},
  author={Wu, Sifan and Xiao, Xi and Ding, Qianggang and Zhao, Peilin and Wei, Ying and Huang, Junzhou},
  journal=P_NIPS,
  volume={33},
  pages={17105--17115},
  year={2020}
}

@article{SDGCN,
  title={Transformer-Based Generative Adversarial Network for Traffic Forecasting}, 
  author={Liu, Bingyi and Yuan, Lingtian and Shao, Xun and Wang, Enshu and Xia, Zhenchang and Han, Weizhen and Wu, Celimuge},
  journal=IEEE_J_CE, 
  year={2025},
  volume={71},
  number={4},
  pages={12288-12299}
}

@inproceedings{TimeGrad,
  title={Autoregressive denoising diffusion models for multivariate probabilistic time series forecasting},
  author={Rasul, Kashif and Seward, Calvin and Schuster, Ingmar and Vollgraf, Roland},
  booktitle=P_ICML,
  pages={8857--8868},
  year={2021},
  organization={PMLR}
}

@article{CSDI,
  title={Csdi: Conditional score-based diffusion models for probabilistic time series imputation},
  author={Tashiro, Yusuke and Song, Jiaming and Song, Yang and Ermon, Stefano},
  journal=P_NIPS,
  volume={34},
  pages={24804--24816},
  year={2021}
}

@article{MoLA,
  title={Mixture of Low Rank Adaptation with Partial Parameter Sharing for Time Series Forecasting},
  author={Pan, Licheng and Chen, Zhichao and Li, Haoxuan and Liu, Guangyi and Xu, Zhijian and Liu, Zhaoran and Wang, Hao and Wei, Ying},
  journal={arXiv preprint arXiv:2505.17872},
  year={2025}
}

@inproceedings{TimeBridge,
  title = 	 {TimeBridge: Non-Stationarity Matters for Long-term Time Series Forecasting},
  author =       {Liu, Peiyuan and Wu, Beiliang and Hu, Yifan and Li, Naiqi and Dai, Tao and Bao, Jigang and Xia, Shu-Tao},
  booktitle = 	 P_ICML,
  pages = 	 {39815--39840},
  year = 	 {2025},
  editor = 	 {Singh, Aarti and Fazel, Maryam and Hsu, Daniel and Lacoste-Julien, Simon and Berkenkamp, Felix and Maharaj, Tegan and Wagstaff, Kiri and Zhu, Jerry},
  volume = 	 {267},
  series = 	 {Proceedings of Machine Learning Research},
  month = 	 {13--19 Jul},
  publisher =    {PMLR}
}

@inproceedings{CFPT,
  title = 	 {CFPT: Empowering Time Series Forecasting through Cross-Frequency Interaction and Periodic-Aware Timestamp Modeling},
  author =       {Kou, Feifei and Wang, Jiahao and Shi, Lei and Yao, Yuhan and Li, Yawen and Zhu, Suguo and Zhang, Zhongbao and Du, Junping},
  booktitle = 	 P_ICML,
  pages = 	 {31634--31647},
  year = 	 {2025},
  editor = 	 {Singh, Aarti and Fazel, Maryam and Hsu, Daniel and Lacoste-Julien, Simon and Berkenkamp, Felix and Maharaj, Tegan and Wagstaff, Kiri and Zhu, Jerry},
  volume = 	 {267},
  series = 	 {Proceedings of Machine Learning Research},
  month = 	 {13--19 Jul},
  publisher =    {PMLR}
}

\newpage
\appendix

\section{Extended Related Work}
\label{sec:extended_related_work}

This section provides an extended review of deep time-series forecasting literature through the lens of autocorrelation modeling. As delineated in the Introduction, existing research primarily addresses two fundamental challenges: (1) devising neural architectures to model autocorrelation in history sequences, and (2) designing learning objectives to accommodate autocorrelation in label sequences.

\subsection{Neural Architectures for History Autocorrelation}
Recent advancements in neural architectures focus on capturing temporal dependencies within the historical input to generate latent representations for forecasting. These approaches can be broadly categorized into non-Transformer and Transformer-based models.

\paragraph{Non-Transformer Models.} This category includes Recurrent Neural Networks (RNNs), Convolutional Neural Networks (CNNs), and Dense Neural Networks (DNNs). 
\ding{182} RNNs naturally model sequential autocorrelation via hidden states. While traditional RNNs (\eg LSTM~\citep{LSTM}, GRU~\citep{GRU}) suffer from efficiency bottlenecks, recent State Space Models (SSMs) like Mamba~\citep{gu2023mamba} and its variants (\eg MixMamba~\citep{MixMamba}) have emerged to efficiently model long-range dependencies with linear complexity.
\ding{183} CNNs capture local autocorrelation through convolutional kernels. To expand the receptive field for long-term forecasting, recent works employ large-kernel designs (\eg ModernTCN~\citep{Moderntcn}) or operate in the frequency domain (\eg TimesNet~\citep{Timesnet}, MICN~\citep{micn}).
\ding{184} DNNs utilizing simple linear layers have shown competitive performance by treating history sequences as feature vectors. Representative models like DLinear~\citep{DLinear} and TiDE~\citep{TiDE} demonstrate that Multi-Layer Perceptrons (MLPs) can effectively model autocorrelation when combined with decomposition or channel-independence strategies.

\paragraph{Transformer-based Models.} Transformers utilize self-attention mechanisms to model high-order autocorrelation by computing pairwise similarities across time steps.
\ding{182} Standard Self-Attention Models often employ tokenization to incorporate local context. For instance, PatchTST~\citep{PatchTST} segments series into patches to retain semantic meaning. A growing trend involves adapting Large Language Models (LLMs)~\citep{wang2026museagent,liu2026memgui,xu2026draftingagentsjusttokens} for time-series (\eg Time-LLM~\citep{Time-LLM}, GPT4TS~\citep{GPT4TS}) via reprogramming or fine-tuning to leverage their pre-trained reasoning capabilities.
\ding{183} Modified Attention Models refine the attention mechanism to reduce complexity or enhance feature extraction. Notable approaches include frequency-domain attention (\eg FedFormer~\citep{fedformer}, FreDF~\citep{wang2025iclrfredf}) and inverted attention mechanisms that model correlations across variates (\eg iTransformer~\citep{itransformer}).

\subsection{Learning Objectives for Label Autocorrelation}
While architectures focus on inputs, learning objectives aim to mitigate the bias arising from label autocorrelation, where future values are statistically dependent on their predecessors. Standard objectives like Mean Squared Error (MSE) assume conditional independence, leading to suboptimal training. Recent objectives mainly address this issue via likelihood estimation, shape alignment, conditional generation, and distribution balancing.

\paragraph{Likelihood Estimation.} These methods aim to refine the negative log-likelihood (NLL) estimation.
\ding{182} Label Transformation: To eliminate autocorrelation bias, methods like FreDF~\citep{wang2025iclrfredf} and Time-o1~\citep{wang2025nipstimeo1} transform labels into frequency or principal component domains, respectively, aiming to whiten the label distribution before applying point-wise losses.
\ding{182} Covariance Modeling: Approaches such as QDF~\citep{wang2026iclrqdf} explicitly model the conditional covariance matrix to approximate the true NLL, often utilizing meta-learning or auxiliary networks.

\paragraph{Shape Alignment.} Recognizing that autocorrelation manifests in morphological shapes, these objectives measure geometric dissimilarity. SoftDTW~\citep{soft-dtw} provides a differentiable relaxation of Dynamic Time Warping. Variants like Dilate~\citep{Dilate} and STRIPE~\citep{STRIPE2} further incorporate temporal distortion penalties to better align forecast shapes with ground truths.

\paragraph{Conditional Generation.} This paradigm reframes forecasting as a generative task to implicitly capture label autocorrelation. Diffusion models (\eg TimeGrad~\citep{TimeGrad}, CSDI~\citep{CSDI}) generate forecasts by learning the conditional distribution $\mathbb{P}(\mathbf{y}|\mathbf{x})$ via iterative denoising. Autoregressive models (\eg DeepAR~\citep{DeepAR}, AutoTimes~\citep{Autotimes}) generate predictions step-by-step to explicitly condition on previous outputs, though often suffering from error accumulation.

\paragraph{Distribution Balancing.} Most relevant to our work, these methods view forecasting as aligning the distribution of forecasts with ground truths.
\ding{182} Adversarial Training: Methods like AST~\citep{AST} and variants~\citep{SDGCN} employ GAN-based discriminators to distinguish between predicted and real sequences, enforcing distributional alignment implicitly.
\ding{183} Discrepancy Minimization: Recent works like DistDF~\citep{wang2026iclrdistdf} explicitly minimize distributional discrepancies (\eg joint-distribution Wasserstein distance).

While distribution balancing methods are theoretically appealing, existing approaches typically enforce moment matching only for limited or predefined functions (\eg first/second moments in DistDF or specific discriminator architectures in AST). As discussed in Section~\ref{subsec:motivation}, they fail to satisfy Imbens' criterion, which requires equivalence across \textbf{any} balancing function. Our proposed KMB-DF addresses this gap by leveraging the Reproducing Kernel Hilbert Space (RKHS) to adaptively select infinite dimensional balancing functions, ensuring sufficient distribution alignment.

\section{Theoretical Justification}\label{sec:theoretical_justification}

In this section, we delve into the theoretical underpinnings of the framework. By exploiting the universality, injectivity, and reproducing properties of kernel functions, we derive formal guarantees for the proposed method and demonstrate how these properties translate into effective balancing.

\begin{lemma}[Autocorrelation bias, Lemma~\ref{lem:bias} in the main text]
    Let $\mathbf{y}\in\mathbb{R}^\mathrm{T}$ be a univariate label sequence with conditional covariance $\mathbf{\Sigma}\in\mathbb{R}^{\mathrm{T}\times\mathrm{T}}$. $\mathcal{E}_\mathrm{MSE}$ in Eq.~(\ref{eq:mse}) is biased against the likelihood of $\mathbf{y}$ unless $\mathbf{\Sigma}$ is diagonal, \ie different steps in $\mathbf{y}$ are conditionally decorrelated given $\mathbf{x}$.
\end{lemma}
\begin{proof}
    The proof can be found in Theorem 3.1 of \citet{wang2025nipstimeo1}.
\end{proof}

\begin{lemma}[Representer theorem]\label{lem:A1}
    Suppose $h(\|f\|): \mathbb{R}_{+} \rightarrow \mathbb{R}$ is a non-decreasing function. The minimizer of an empirical risk functional regularized by $h(\|f\|)$ admits the form: $f^*(\cdot)=\sum_{n=1}^\mathrm{N} \alpha_n \mathcal{K}(\cdot, \mathbf{z}_n)$, where $\mathcal{K}$ is the kernel function associated with the RKHS $\mathcal{H}$. 
\end{lemma}
\begin{proof}
    The proof can be found in Theorem 6.11 of \citet{mohri2018foundations}.
\end{proof}

\begin{theorem}
    \label{thm:kernel_align_app}
    Suppose $\mathcal{H}$ is the RKHS of a exponential kernel $\mathcal{K}$, $f^* \in \mathcal{H}$ is the discrimination function that maximizes the discrepancy between two distributions $\mathbb{P}$ and $\mathbb{Q}$. If $f^*$ lies within the linear span of a finite kernel function set: $f^*\in \{\mathcal{K}(\cdot, \mathbf{z}_k)\}_{k=1}^\mathrm{K}$; then $\mathbb{P}$ and $\mathbb{Q}$ are balanced if $\forall k \in [\mathrm{K}]:$ $\mathbb{E}_{\mathbf{z}\in\mathbb{P}}[\mathcal{K}(\mathbf{z}, \mathbf{z}_k)] = \mathbb{E}_{\hat{\mathbf{z}}\in\mathbb{Q}}[\mathcal{K}(\hat{\mathbf{z}}, \mathbf{z}_k)]$.
\end{theorem}

\begin{proof}

We start by clarifying the maximum mean discrepancy (MMD) as our discrepancy measure between two distributions. Specifically, given a universal kernel $\mathcal{K}$ (\eg exponential kernel), the MMD is defiend as:
\begin{equation}
    \text{MMD}^2(\mathbb{P}, \mathbb{Q}) = \sup_{f \in \mathcal{H}, \|f\|_{\mathcal{H}} \leq 1} \left| \mathbb{E}_{\mathbf{z} \sim \mathbb{P}}[f(\mathbf{z})] - \mathbb{E}_{\hat{\mathbf{z}} \sim \mathbb{Q}}[f(\hat{\mathbf{z}})] \right|^2.
\end{equation} 

Let $f^*$ be the discrimination function that maximizes the mean discrepancy between the two distributions $\mathbb{P}$ and $\mathbb{Q}$, i.e., the function that achieves the supremum in the $\text{MMD}^2$ definition. By Lemma~\ref{lem:A1}, functions in $\mathcal{H}$ can be represented as linear combinations of kernel functions. 
On the basis of this, by the assumption that $f^*$ lies within the linear span of a finite kernel function set, there exist $\alpha_1, \dots, \alpha_\mathrm{K}$ such that:
\begin{equation}
    f^*(\cdot) = \sum_{k=1}^\mathrm{K} \alpha_k \mathcal{K}(\cdot, \mathbf{z}_k).
\end{equation}

On this basis, we calculate the informativeness score where $f^*$ serves as the balancing function:
\begin{equation}
    \begin{aligned}
        \delta(f^*) &= \mathbb{E}_{\mathbf{z}\sim\mathbb{P}}[f^*(\mathbf{z})] - \mathbb{E}_{\hat{\mathbf{z}}\sim\mathbb{Q}}[f^*(\hat{\mathbf{z}})] \\
        &= \mathbb{E}_{\mathbf{z}\sim\mathbb{P}}\left[\sum_{k=1}^\mathrm{K} \alpha_k \mathcal{K}(\mathbf{z}, \mathbf{z}_k)\right] - \mathbb{E}_{\hat{\mathbf{z}}\sim\mathbb{Q}}\left[\sum_{k=1}^\mathrm{K} \alpha_k \mathcal{K}(\hat{\mathbf{z}}, \mathbf{z}_k)\right]\\
        &\overset{(a)}{=} \sum_{k=1}^\mathrm{K} \alpha_k \left( \mathbb{E}_{\mathbf{z}\sim\mathbb{P}}[\mathcal{K}(\mathbf{z}, \mathbf{z}_k)] - \mathbb{E}_{\hat{\mathbf{z}}\sim\mathbb{Q}}[\mathcal{K}(\hat{\mathbf{z}}, \mathbf{z}_k)] \right)\\
        &\overset{(b)}{=} 0,
    \end{aligned}
\end{equation}
where (a) exploits the linearity of the expectation operator to swap the summation and expectation; (b) applies the assumption that $\mathbb{E}_{\mathbf{z}\in\mathbb{P}}[\mathcal{K}(\mathbf{z}, \mathbf{z}_k)] = \mathbb{E}_{\hat{\mathbf{z}}\in\mathbb{Q}}[\mathcal{K}(\hat{\mathbf{z}}, \mathbf{z}_k)]$ for all $k=1,\dots,\mathrm{K}$.

Since $f^*$ is the function that maximizes the mean discrepancy between $\mathbb{P}$ and $\mathbb{Q}$, the maximum mean discrepancy is zero, \ie $\text{MMD}(\mathbb{P}, \mathbb{Q}) = 0$. Moreover, since $\mathcal{K}$ is a universal kernel, $\text{MMD}^2(\mathbb{P}, \mathbb{Q}) = 0$ implies $\mathbb{P} = \mathbb{Q}$~\citep{kmmd}, which completes the proof.
\end{proof}

\section{Reproduction Details}\label{sec:reproduce}
\subsection{Dataset descriptions}\label{sec:dataset}

\begin{table}
\caption{Dataset description. }\label{tab:dataset}
\centering
\renewcommand{\multirowsetup}{\centering}
\setlength{\tabcolsep}{11pt}
\small
\begin{threeparttable}
\begin{tabular}{llllll}
    \toprule
    Dataset & D & Forecast length & Train / validation / test & Frequency& Domain \\
    \toprule
     ETTh1 & 7 & 96, 192, 336, 720 & 8545/2881/2881 & Hourly & Electricity\\
     \midrule
     ETTh2 & 7 & 96, 192, 336, 720 & 8545/2881/2881 & Hourly & Electricity\\
     \midrule
     ETTm1 & 7 & 96, 192, 336, 720 & 34465/11521/11521 & 15min & Electricity\\
     \midrule
     ETTm2 & 7 & 96, 192, 336, 720 & 34465/11521/11521 & 15min & Electricity\\
    \midrule
    ECL & 321 & 96, 192, 336, 720 & 18317/2633/5261 & Hourly & Electricity \\
    \midrule
    Weather & 21 & 96, 192, 336, 720 & 36792/5271/10540 & 10min & Weather\\
    \midrule
    M5 & 10 & 8, 12, 20, 28 & 1782/21/21 & Daily & Sale \\
    \bottomrule
\end{tabular}
    \begin{tablenotes}
    \item  \scriptsize \textit{Note}:  \textit{D} denotes the number of variates. \emph{Frequency} denotes the sampling interval of time points. \emph{Train, Validation, Test} denotes the number of samples employed in each split. The taxonomy aligns with~\citep{Timesnet}.
    \end{tablenotes}
\end{threeparttable}
\end{table}

Our empirical evaluation is conducted on a diverse collection of widely-used time series forecasting benchmarks. Each dataset presents distinct characteristics in terms of dimensionality and temporal resolution. A summary is provided in \autoref{tab:dataset}.

\begin{itemize}[leftmargin=*]
    \item \textbf{ETT}~\citep{Informer}: Contains seven metrics related to electricity transformers, recorded from July 2016 to July 2018. It is divided into four subsets based on sampling frequency: ETTh1 and ETTh2 (hourly), and ETTm1 and ETTm2 (every 15 minutes).
    \item \textbf{ECL}~\citep{Autoformer}: Features the hourly electricity consumption of 321 clients.
    \item \textbf{Weather}~\citep{Autoformer}: Comprises 21 meteorological variables from the Max Planck Biogeochemistry Institute's weather station, captured every 10 minutes throughout 2020.
    \item \textbf{M5}~\citep{M5}: Comprises 3049 individual products from 3 categories and 7 departments, sold in 10 stores in 3 states.
\end{itemize}

Following established protocols~\citep{qiutfb, itransformer}, all datasets are chronologically partitioned into training, validation, and test sets. For the ETT, Weather, and ECL, and datasets, we use a fixed history sequence length of 96 and evaluate performance across four prediction horizons with lengths of 96, 192, 336, and 720. For the M5 datasets, we also use an historical length of 96 but evaluate on shorter prediction horizons of 8, 12, 20, and 28 steps. During the final evaluation on the test set, we ensure that no data is discarded from the last batch: a technique  referred to as the \textit{dropping-last trick} is disabled throughout our experiments.

\subsection{Implementation details of model training}
To ensure a rigorous and fair comparison, we strictly adhere to standard evaluation protocols established in recent literature~\citep{Timesnet,itransformer}. We reproduce all baseline methods using their official implementations, primarily integrating codes from the DistDF~\citep{wang2026iclrdistdf} and CFPT~\citep{CFPT} repositories. All models are optimized using the Adam optimizer~\citep{Adam} with a fixed batch size of 32. Following~\citet{qiutfb}, we disable the ``drop-last'' operation in data loaders during the testing phase to prevent data leakage and ensure metrics are calculated over the complete test set. To prevent overfitting, we employ an early stopping mechanism that terminates training if the validation loss fails to improve for 15 consecutive epochs. The initial learning rate is tuned for each dataset-model combination via a grid search over $\{5 \times 10^{-3}, 2 \times 10^{-3}, 10^{-3}, 5 \times 10^{-4}, 2 \times 10^{-4}, 10^{-4}, 5 \times 10^{-5}\}$, selecting the value that yields the minimal MSE on the validation set. All experiments are conducted on a computational cluster equipped with Intel(R) Xeon(R) Platinum 8463B CPUs and NVIDIA RTX H800 GPUs.

When integrating KMB-DF with forecast backbones (\eg CFPT, TimeBridge), we explicitly retain the original models' architectural hyperparameters as reported in their respective benchmarks to ensure that performance gains are attributed solely to the proposed objective. Consequently, our hyperparameter tuning is exclusively focused on the components of KMB-DF: the penalty strength $\alpha$, the tolerated imbalance margin $\mathrm{C}$, and the number of balancing functions $\mathrm{K}$. Specifically, $\alpha$ is tuned within the interval $(0, 1]$ to balance the trade-off between the standard forecasting loss and the moment matching penalty. The margin $\mathrm{C}$ is selected from the set $\{0.0005, 0.001, 0.005, 0.01, 0.05\}$ to control the strictness of distribution alignment, while the number of balancing functions $\mathrm{K}$ is searched within the integer range $[1, 6]$. We utilize the exponential kernel as the default instantiation for the balancing functions.

\begin{table}[t]
\caption{Full results on the multi-step forecasting task. The length of history window is set to 96 for all baselines. \texttt{Avg} indicates the results averaged over forecasting lengths: T=96, 192, 336 and 720 for ETT, ECL, and Weather; T=8, 12, 20 and 28 for M5.}\label{tab:multistep_app_full}
\renewcommand{\arraystretch}{0.8}
\setlength{\tabcolsep}{2pt}
\scriptsize
\centering
\renewcommand{\multirowsetup}{\centering}
\begin{tabular}{c|c|cc|cc|cc|cc|cc|cc|cc|cc|cc|cc|cc}
    \toprule
    \multicolumn{2}{l}{\multirow{2}{*}{\rotatebox{0}{\scaleb{Loss}}}} & 
    \multicolumn{2}{c}{\rotatebox{0}{\scaleb{\textbf{KMB-DF}}}} &
    \multicolumn{2}{c}{\rotatebox{0}{\scaleb{QDF}}} &
    \multicolumn{2}{c}{\rotatebox{0}{\scaleb{DistDF}}} &
    \multicolumn{2}{c}{\rotatebox{0}{\scaleb{Time-o1}}} &
    \multicolumn{2}{c}{\rotatebox{0}{\scaleb{FreDF}}} &
    \multicolumn{2}{c}{\rotatebox{0}{\scaleb{DBLoss}}} &
    \multicolumn{2}{c}{\rotatebox{0}{\scaleb{Koopman}}} &
    \multicolumn{2}{c}{\rotatebox{0}{\scaleb{GDTW}}} &
    \multicolumn{2}{c}{\rotatebox{0}{\scaleb{Dilate}}} &
    \multicolumn{2}{c}{\rotatebox{0}{\scaleb{Soft-DTW}}} &
    \multicolumn{2}{c}{\rotatebox{0}{\scaleb{MSE}}} \\
    \multicolumn{2}{c}{} &
    \multicolumn{2}{c}{\scaleb{\textbf{(Ours)}}} & 
    \multicolumn{2}{c}{\scaleb{(2025)}} & 
    \multicolumn{2}{c}{\scaleb{(2025)}} & 
    \multicolumn{2}{c}{\scaleb{(2025)}} & 
    \multicolumn{2}{c}{\scaleb{(2025)}} & 
    \multicolumn{2}{c}{\scaleb{(2025)}} & 
    \multicolumn{2}{c}{\scaleb{(2021)}} & 
    \multicolumn{2}{c}{\scaleb{(2021)}} &
    \multicolumn{2}{c}{\scaleb{(2019)}} &
    \multicolumn{2}{c}{\scaleb{(2017)}} &
    \multicolumn{2}{c}{\scaleb{(2002)}} \\
    \cmidrule(lr){3-4} \cmidrule(lr){5-6}\cmidrule(lr){7-8} \cmidrule(lr){9-10}\cmidrule(lr){11-12} \cmidrule(lr){13-14} \cmidrule(lr){15-16} \cmidrule(lr){17-18} \cmidrule(lr){19-20} \cmidrule(lr){21-22} \cmidrule(lr){23-24}
    \multicolumn{2}{l}{\rotatebox{0}{\scaleb{Metrics}}}  & \scalea{MSE} & \scalea{MAE}  & \scalea{MSE} & \scalea{MAE}  & \scalea{MSE} & \scalea{MAE}  & \scalea{MSE} & \scalea{MAE}  & \scalea{MSE} & \scalea{MAE}  & \scalea{MSE} & \scalea{MAE} & \scalea{MSE} & \scalea{MAE} & \scalea{MSE} & \scalea{MAE} & \scalea{MSE} & \scalea{MAE} & \scalea{MSE} & \scalea{MAE} & \scalea{MSE} & \scalea{MAE} \\
    \toprule
    
\multirow{5}{*}{{\rotatebox{90}{\scalebox{0.95}{ETTm1}}}}
& 96 & \scalea{\bst{0.315}} & \scalea{\bst{0.352}} & \scalea{0.318} & \scalea{0.359} & \scalea{0.318} & \scalea{0.355} & \scalea{0.321} & \scalea{\subbst{0.353}} & \scalea{0.320} & \scalea{0.354} & \scalea{\subbst{0.318}} & \scalea{0.354} & \scalea{0.326} & \scalea{0.363} & \scalea{0.324} & \scalea{0.362} & \scalea{0.323} & \scalea{0.361} & \scalea{0.322} & \scalea{0.359} & \scalea{0.321} & \scalea{0.358}  \\
& 192 & \scalea{\bst{0.351}} & \scalea{\bst{0.379}} & \scalea{\subbst{0.355}} & \scalea{0.379} & \scalea{0.356} & \scalea{0.380} & \scalea{0.358} & \scalea{0.380} & \scalea{0.356} & \scalea{\subbst{0.379}} & \scalea{0.359} & \scalea{0.379} & \scalea{0.359} & \scalea{0.383} & \scalea{0.368} & \scalea{0.392} & \scalea{0.363} & \scalea{0.385} & \scalea{0.366} & \scalea{0.385} & \scalea{0.357} & \scalea{0.382}  \\
& 336 & \scalea{\bst{0.381}} & \scalea{\bst{0.400}} & \scalea{0.384} & \scalea{0.401} & \scalea{0.383} & \scalea{0.401} & \scalea{0.386} & \scalea{0.403} & \scalea{0.386} & \scalea{0.401} & \scalea{0.389} & \scalea{0.401} & \scalea{\subbst{0.383}} & \scalea{\subbst{0.401}} & \scalea{0.405} & \scalea{0.421} & \scalea{0.384} & \scalea{0.401} & \scalea{0.399} & \scalea{0.409} & \scalea{0.387} & \scalea{0.401}  \\
& 720 & \scalea{\bst{0.442}} & \scalea{\bst{0.433}} & \scalea{\subbst{0.443}} & \scalea{0.435} & \scalea{0.444} & \scalea{0.434} & \scalea{0.447} & \scalea{0.436} & \scalea{0.444} & \scalea{0.434} & \scalea{0.450} & \scalea{\subbst{0.434}} & \scalea{0.446} & \scalea{0.436} & \scalea{0.496} & \scalea{0.473} & \scalea{0.446} & \scalea{0.437} & \scalea{0.490} & \scalea{0.461} & \scalea{0.446} & \scalea{0.435}  \\
\cmidrule(lr){2-22}
& Avg & \scalea{\bst{0.372}} & \scalea{\bst{0.391}} & \scalea{\subbst{0.375}} & \scalea{0.393} & \scalea{0.375} & \scalea{0.393} & \scalea{0.378} & \scalea{0.393} & \scalea{0.376} & \scalea{\subbst{0.392}} & \scalea{0.379} & \scalea{0.392} & \scalea{0.378} & \scalea{0.395} & \scalea{0.398} & \scalea{0.412} & \scalea{0.379} & \scalea{0.396} & \scalea{0.394} & \scalea{0.403} & \scalea{0.378} & \scalea{0.394}  \\
\midrule

\multirow{5}{*}{{\rotatebox{90}{\scalebox{0.95}{ETTm2}}}}
& 96 & \scalea{\bst{0.166}} & \scalea{\bst{0.249}} & \scalea{\subbst{0.168}} & \scalea{\subbst{0.250}} & \scalea{0.168} & \scalea{0.251} & \scalea{0.173} & \scalea{0.251} & \scalea{0.174} & \scalea{0.251} & \scalea{0.175} & \scalea{0.250} & \scalea{0.182} & \scalea{0.262} & \scalea{0.176} & \scalea{0.257} & \scalea{0.172} & \scalea{0.253} & \scalea{0.176} & \scalea{0.256} & \scalea{0.175} & \scalea{0.256}  \\
& 192 & \scalea{\bst{0.227}} & \scalea{\bst{0.288}} & \scalea{\subbst{0.232}} & \scalea{0.292} & \scalea{0.232} & \scalea{0.293} & \scalea{0.233} & \scalea{0.291} & \scalea{0.233} & \scalea{0.290} & \scalea{0.233} & \scalea{\subbst{0.289}} & \scalea{0.235} & \scalea{0.294} & \scalea{0.243} & \scalea{0.305} & \scalea{0.236} & \scalea{0.296} & \scalea{0.243} & \scalea{0.299} & \scalea{0.234} & \scalea{0.295}  \\
& 336 & \scalea{\bst{0.284}} & \scalea{\subbst{0.327}} & \scalea{\subbst{0.288}} & \scalea{0.329} & \scalea{0.289} & \scalea{0.330} & \scalea{0.290} & \scalea{0.329} & \scalea{0.290} & \scalea{\bst{0.327}} & \scalea{0.294} & \scalea{0.329} & \scalea{0.291} & \scalea{0.331} & \scalea{0.303} & \scalea{0.344} & \scalea{0.290} & \scalea{0.331} & \scalea{0.306} & \scalea{0.342} & \scalea{0.289} & \scalea{0.330}  \\
& 720 & \scalea{\bst{0.382}} & \scalea{\subbst{0.385}} & \scalea{0.386} & \scalea{0.387} & \scalea{0.385} & \scalea{0.387} & \scalea{0.386} & \scalea{\bst{0.385}} & \scalea{0.386} & \scalea{0.386} & \scalea{0.391} & \scalea{0.387} & \scalea{0.388} & \scalea{0.390} & \scalea{0.434} & \scalea{0.419} & \scalea{0.393} & \scalea{0.390} & \scalea{0.438} & \scalea{0.422} & \scalea{\subbst{0.385}} & \scalea{0.389}  \\
\cmidrule(lr){2-22}
& Avg & \scalea{\bst{0.265}} & \scalea{\bst{0.312}} & \scalea{\subbst{0.268}} & \scalea{0.315} & \scalea{0.269} & \scalea{0.315} & \scalea{0.271} & \scalea{0.314} & \scalea{0.271} & \scalea{\subbst{0.313}} & \scalea{0.273} & \scalea{0.314} & \scalea{0.274} & \scalea{0.319} & \scalea{0.289} & \scalea{0.332} & \scalea{0.273} & \scalea{0.317} & \scalea{0.291} & \scalea{0.330} & \scalea{0.271} & \scalea{0.317}  \\
\midrule

\multirow{5}{*}{{\rotatebox{90}{\scalebox{0.95}{ETTh1}}}}
& 96 & \scalea{\bst{0.367}} & \scalea{\subbst{0.390}} & \scalea{0.373} & \scalea{0.392} & \scalea{0.373} & \scalea{0.391} & \scalea{\subbst{0.369}} & \scalea{0.392} & \scalea{0.373} & \scalea{\bst{0.389}} & \scalea{0.375} & \scalea{0.391} & \scalea{0.374} & \scalea{0.393} & \scalea{0.377} & \scalea{0.395} & \scalea{0.376} & \scalea{0.394} & \scalea{0.377} & \scalea{0.394} & \scalea{0.373} & \scalea{0.391}  \\
& 192 & \scalea{\bst{0.422}} & \scalea{0.420} & \scalea{0.426} & \scalea{0.421} & \scalea{0.426} & \scalea{\subbst{0.420}} & \scalea{\subbst{0.424}} & \scalea{0.424} & \scalea{0.424} & \scalea{\bst{0.418}} & \scalea{0.428} & \scalea{0.420} & \scalea{0.426} & \scalea{0.424} & \scalea{0.434} & \scalea{0.427} & \scalea{0.430} & \scalea{0.423} & \scalea{0.436} & \scalea{0.427} & \scalea{0.427} & \scalea{0.421}  \\
& 336 & \scalea{\bst{0.460}} & \scalea{\bst{0.438}} & \scalea{0.464} & \scalea{0.441} & \scalea{0.468} & \scalea{0.441} & \scalea{\subbst{0.464}} & \scalea{0.439} & \scalea{0.465} & \scalea{\subbst{0.439}} & \scalea{0.466} & \scalea{0.441} & \scalea{0.466} & \scalea{0.440} & \scalea{0.486} & \scalea{0.457} & \scalea{0.470} & \scalea{0.442} & \scalea{0.479} & \scalea{0.451} & \scalea{0.466} & \scalea{0.441}  \\
& 720 & \scalea{\bst{0.447}} & \scalea{\bst{0.455}} & \scalea{0.473} & \scalea{0.464} & \scalea{0.467} & \scalea{0.462} & \scalea{\subbst{0.462}} & \scalea{\subbst{0.460}} & \scalea{0.472} & \scalea{0.468} & \scalea{0.474} & \scalea{0.462} & \scalea{0.483} & \scalea{0.469} & \scalea{0.512} & \scalea{0.497} & \scalea{0.489} & \scalea{0.478} & \scalea{0.551} & \scalea{0.510} & \scalea{0.477} & \scalea{0.468}  \\
\cmidrule(lr){2-22}
& Avg & \scalea{\bst{0.424}} & \scalea{\bst{0.426}} & \scalea{0.434} & \scalea{0.429} & \scalea{0.434} & \scalea{0.428} & \scalea{\subbst{0.430}} & \scalea{0.429} & \scalea{0.434} & \scalea{\subbst{0.428}} & \scalea{0.436} & \scalea{0.429} & \scalea{0.437} & \scalea{0.431} & \scalea{0.452} & \scalea{0.444} & \scalea{0.441} & \scalea{0.434} & \scalea{0.461} & \scalea{0.445} & \scalea{0.436} & \scalea{0.430}  \\
\midrule

\multirow{5}{*}{{\rotatebox{90}{\scalebox{0.95}{ETTh2}}}}
& 96 & \scalea{\bst{0.280}} & \scalea{\bst{0.331}} & \scalea{0.288} & \scalea{0.337} & \scalea{0.287} & \scalea{0.337} & \scalea{\subbst{0.286}} & \scalea{0.336} & \scalea{0.289} & \scalea{0.337} & \scalea{0.288} & \scalea{\subbst{0.335}} & \scalea{0.287} & \scalea{0.338} & \scalea{0.289} & \scalea{0.340} & \scalea{0.287} & \scalea{0.338} & \scalea{0.292} & \scalea{0.342} & \scalea{0.287} & \scalea{0.337}  \\
& 192 & \scalea{\bst{0.358}} & \scalea{\bst{0.381}} & \scalea{0.367} & \scalea{0.390} & \scalea{\subbst{0.364}} & \scalea{0.391} & \scalea{0.367} & \scalea{0.384} & \scalea{0.365} & \scalea{\subbst{0.382}} & \scalea{0.367} & \scalea{0.385} & \scalea{0.365} & \scalea{0.390} & \scalea{0.377} & \scalea{0.399} & \scalea{0.368} & \scalea{0.389} & \scalea{0.380} & \scalea{0.398} & \scalea{0.368} & \scalea{0.390}  \\
& 336 & \scalea{\bst{0.411}} & \scalea{\bst{0.425}} & \scalea{0.412} & \scalea{0.427} & \scalea{\subbst{0.412}} & \scalea{\subbst{0.426}} & \scalea{0.414} & \scalea{0.427} & \scalea{0.414} & \scalea{0.427} & \scalea{0.420} & \scalea{0.427} & \scalea{0.413} & \scalea{0.429} & \scalea{0.422} & \scalea{0.432} & \scalea{0.413} & \scalea{0.427} & \scalea{0.455} & \scalea{0.443} & \scalea{0.416} & \scalea{0.428}  \\
& 720 & \scalea{\bst{0.396}} & \scalea{\bst{0.426}} & \scalea{0.401} & \scalea{0.428} & \scalea{\subbst{0.398}} & \scalea{\subbst{0.427}} & \scalea{0.402} & \scalea{0.428} & \scalea{0.406} & \scalea{0.430} & \scalea{0.411} & \scalea{0.430} & \scalea{0.405} & \scalea{0.430} & \scalea{0.457} & \scalea{0.463} & \scalea{0.415} & \scalea{0.437} & \scalea{0.445} & \scalea{0.453} & \scalea{0.415} & \scalea{0.436}  \\
\cmidrule(lr){2-22}
& Avg & \scalea{\bst{0.361}} & \scalea{\bst{0.391}} & \scalea{0.367} & \scalea{0.396} & \scalea{\subbst{0.365}} & \scalea{0.395} & \scalea{0.367} & \scalea{\subbst{0.394}} & \scalea{0.368} & \scalea{0.394} & \scalea{0.371} & \scalea{0.394} & \scalea{0.368} & \scalea{0.397} & \scalea{0.386} & \scalea{0.409} & \scalea{0.371} & \scalea{0.398} & \scalea{0.393} & \scalea{0.409} & \scalea{0.372} & \scalea{0.398}  \\
\midrule

\multirow{5}{*}{{\rotatebox{90}{\scalebox{0.95}{ECL}}}}
& 96 & \scalea{\bst{0.136}} & \scalea{\bst{0.231}} & \scalea{0.137} & \scalea{0.232} & \scalea{0.137} & \scalea{0.232} & \scalea{0.141} & \scalea{0.235} & \scalea{0.139} & \scalea{0.232} & \scalea{0.141} & \scalea{0.233} & \scalea{0.138} & \scalea{0.232} & \scalea{0.882} & \scalea{0.777} & \scalea{0.137} & \scalea{0.232} & \scalea{\subbst{0.136}} & \scalea{\subbst{0.232}} & \scalea{0.137} & \scalea{0.232}  \\
& 192 & \scalea{\bst{0.152}} & \scalea{\bst{0.245}} & \scalea{0.153} & \scalea{0.246} & \scalea{\subbst{0.152}} & \scalea{\subbst{0.246}} & \scalea{0.154} & \scalea{0.246} & \scalea{0.156} & \scalea{0.246} & \scalea{0.155} & \scalea{0.246} & \scalea{0.153} & \scalea{0.247} & \scalea{0.853} & \scalea{0.762} & \scalea{0.154} & \scalea{0.247} & \scalea{0.154} & \scalea{0.247} & \scalea{0.153} & \scalea{0.247}  \\
& 336 & \scalea{\bst{0.168}} & \scalea{0.265} & \scalea{0.168} & \scalea{0.264} & \scalea{0.168} & \scalea{0.265} & \scalea{\subbst{0.168}} & \scalea{\subbst{0.263}} & \scalea{0.169} & \scalea{\bst{0.263}} & \scalea{0.170} & \scalea{0.265} & \scalea{0.169} & \scalea{0.267} & \scalea{0.809} & \scalea{0.711} & \scalea{0.169} & \scalea{0.267} & \scalea{0.169} & \scalea{0.267} & \scalea{0.168} & \scalea{0.266}  \\
& 720 & \scalea{\bst{0.197}} & \scalea{0.291} & \scalea{0.199} & \scalea{0.293} & \scalea{0.201} & \scalea{0.294} & \scalea{\subbst{0.198}} & \scalea{\subbst{0.291}} & \scalea{0.199} & \scalea{\bst{0.290}} & \scalea{0.203} & \scalea{0.293} & \scalea{0.201} & \scalea{0.296} & \scalea{0.900} & \scalea{0.769} & \scalea{0.202} & \scalea{0.297} & \scalea{0.203} & \scalea{0.298} & \scalea{0.199} & \scalea{0.294}  \\
\cmidrule(lr){2-22}
& Avg & \scalea{\bst{0.163}} & \scalea{\subbst{0.258}} & \scalea{\subbst{0.164}} & \scalea{0.259} & \scalea{0.164} & \scalea{0.259} & \scalea{0.165} & \scalea{0.259} & \scalea{0.166} & \scalea{\bst{0.258}} & \scalea{0.167} & \scalea{0.259} & \scalea{0.165} & \scalea{0.260} & \scalea{0.861} & \scalea{0.755} & \scalea{0.165} & \scalea{0.260} & \scalea{0.165} & \scalea{0.261} & \scalea{0.165} & \scalea{0.260}  \\
\midrule

\multirow{5}{*}{{\rotatebox{90}{\scalebox{0.95}{Weather}}}}
& 96 & \scalea{\bst{0.152}} & \scalea{\bst{0.196}} & \scalea{0.155} & \scalea{0.200} & \scalea{0.155} & \scalea{0.200} & \scalea{\subbst{0.153}} & \scalea{\subbst{0.197}} & \scalea{0.154} & \scalea{0.197} & \scalea{0.160} & \scalea{0.198} & \scalea{0.173} & \scalea{0.215} & \scalea{0.158} & \scalea{0.206} & \scalea{0.173} & \scalea{0.215} & \scalea{0.177} & \scalea{0.220} & \scalea{0.156} & \scalea{0.201}  \\
& 192 & \scalea{\bst{0.204}} & \scalea{\bst{0.242}} & \scalea{0.205} & \scalea{0.244} & \scalea{0.205} & \scalea{0.245} & \scalea{0.205} & \scalea{0.243} & \scalea{\subbst{0.204}} & \scalea{\subbst{0.242}} & \scalea{0.208} & \scalea{0.243} & \scalea{0.205} & \scalea{0.245} & \scalea{0.209} & \scalea{0.250} & \scalea{0.205} & \scalea{0.246} & \scalea{0.209} & \scalea{0.245} & \scalea{0.205} & \scalea{0.245}  \\
& 336 & \scalea{\bst{0.260}} & \scalea{\bst{0.284}} & \scalea{0.261} & \scalea{0.285} & \scalea{0.261} & \scalea{0.286} & \scalea{0.261} & \scalea{\subbst{0.284}} & \scalea{\subbst{0.261}} & \scalea{0.284} & \scalea{0.267} & \scalea{0.286} & \scalea{0.262} & \scalea{0.286} & \scalea{0.270} & \scalea{0.296} & \scalea{0.262} & \scalea{0.287} & \scalea{0.278} & \scalea{0.296} & \scalea{0.261} & \scalea{0.286}  \\
& 720 & \scalea{\bst{0.339}} & \scalea{\bst{0.337}} & \scalea{0.343} & \scalea{0.338} & \scalea{0.343} & \scalea{0.341} & \scalea{0.343} & \scalea{0.340} & \scalea{\subbst{0.341}} & \scalea{\subbst{0.338}} & \scalea{0.346} & \scalea{0.338} & \scalea{0.345} & \scalea{0.342} & \scalea{0.354} & \scalea{0.352} & \scalea{0.345} & \scalea{0.342} & \scalea{0.385} & \scalea{0.363} & \scalea{0.343} & \scalea{0.339}  \\
\cmidrule(lr){2-22}
& Avg & \scalea{\bst{0.239}} & \scalea{\bst{0.265}} & \scalea{0.241} & \scalea{0.267} & \scalea{0.241} & \scalea{0.268} & \scalea{0.240} & \scalea{0.266} & \scalea{\subbst{0.240}} & \scalea{\subbst{0.265}} & \scalea{0.245} & \scalea{0.266} & \scalea{0.246} & \scalea{0.272} & \scalea{0.248} & \scalea{0.276} & \scalea{0.246} & \scalea{0.273} & \scalea{0.262} & \scalea{0.281} & \scalea{0.241} & \scalea{0.267}  \\
\midrule

\multirow{5}{*}{{\rotatebox{90}{\scalebox{0.95}{M5}}}}
& 8 & \scalea{\bst{0.150}} & \scalea{\bst{0.307}} & \scalea{0.152} & \scalea{0.309} & \scalea{0.152} & \scalea{0.308} & \scalea{0.159} & \scalea{0.318} & \scalea{\subbst{0.151}} & \scalea{\subbst{0.308}} & \scalea{0.152} & \scalea{0.308} & \scalea{0.154} & \scalea{0.311} & \scalea{0.153} & \scalea{0.312} & \scalea{0.154} & \scalea{0.312} & \scalea{0.154} & \scalea{0.312} & \scalea{0.151} & \scalea{0.309}  \\
& 12 & \scalea{\bst{0.151}} & \scalea{\bst{0.311}} & \scalea{\subbst{0.152}} & \scalea{0.311} & \scalea{0.154} & \scalea{\subbst{0.311}} & \scalea{0.164} & \scalea{0.325} & \scalea{0.155} & \scalea{0.313} & \scalea{0.155} & \scalea{0.312} & \scalea{0.156} & \scalea{0.314} & \scalea{0.160} & \scalea{0.319} & \scalea{0.161} & \scalea{0.319} & \scalea{0.160} & \scalea{0.319} & \scalea{0.157} & \scalea{0.315}  \\
& 20 & \scalea{\bst{0.153}} & \scalea{\bst{0.311}} & \scalea{0.155} & \scalea{\subbst{0.313}} & \scalea{\subbst{0.154}} & \scalea{0.313} & \scalea{0.159} & \scalea{0.317} & \scalea{0.166} & \scalea{0.328} & \scalea{0.162} & \scalea{0.321} & \scalea{0.158} & \scalea{0.318} & \scalea{0.163} & \scalea{0.321} & \scalea{0.165} & \scalea{0.326} & \scalea{0.164} & \scalea{0.326} & \scalea{0.165} & \scalea{0.324}  \\
& 28 & \scalea{\bst{0.154}} & \scalea{\bst{0.309}} & \scalea{0.165} & \scalea{0.320} & \scalea{0.162} & \scalea{0.322} & \scalea{0.162} & \scalea{0.316} & \scalea{0.199} & \scalea{0.361} & \scalea{0.181} & \scalea{0.343} & \scalea{0.164} & \scalea{0.323} & \scalea{0.164} & \scalea{0.319} & \scalea{\subbst{0.156}} & \scalea{\subbst{0.311}} & \scalea{0.169} & \scalea{0.332} & \scalea{0.199} & \scalea{0.361}  \\
\cmidrule(lr){2-22}
& Avg & \scalea{\bst{0.152}} & \scalea{\bst{0.309}} & \scalea{0.156} & \scalea{\subbst{0.313}} & \scalea{\subbst{0.156}} & \scalea{0.313} & \scalea{0.161} & \scalea{0.319} & \scalea{0.168} & \scalea{0.327} & \scalea{0.162} & \scalea{0.321} & \scalea{0.158} & \scalea{0.317} & \scalea{0.160} & \scalea{0.318} & \scalea{0.159} & \scalea{0.317} & \scalea{0.162} & \scalea{0.322} & \scalea{0.168} & \scalea{0.327}  \\
\midrule

\multicolumn{2}{c|}{\scalea{{$1^{\text{st}}$ Count}}} & \scalea{\bst{35}} & \scalea{\bst{28}} & \scalea{0} & \scalea{0} & \scalea{0} & \scalea{0} & \scalea{0} & \scalea{\bst{1}} & \scalea{0} & \scalea{\bst{6}} & \scalea{0} & \scalea{0} & \scalea{0} & \scalea{0} & \scalea{0} & \scalea{0} & \scalea{0} & \scalea{0} & \scalea{0} & \scalea{0} & \scalea{0} & \scalea{0} \\
    \bottomrule
\end{tabular}
\end{table}

\section{More Experimental Results}\label{sec:results_app}

\subsection{Overall performance}\label{sec:overall_app}
Additional experimental results of overall performance are available in  \autoref{tab:multistep_app_full}, where the performance given different $\T$ is reported.

\subsection{Showcase}\label{sec:showcase_app}
Additional experimental results of showcases on ETTm2 dataset are available in \autoref{fig:pred_app}.

\begin{figure}[t]
\begin{center}
\subfigure[ETTm2 snapshot 1.]{\includegraphics[width=0.24\linewidth]{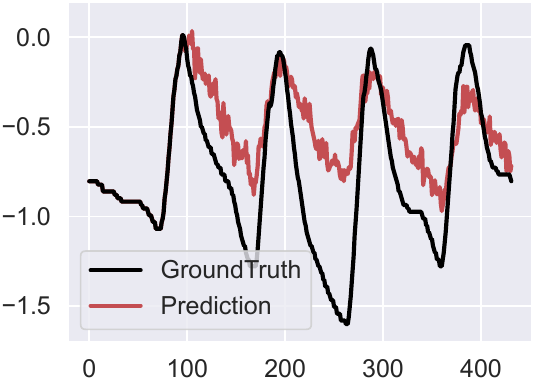}
\includegraphics[width=0.24\linewidth]{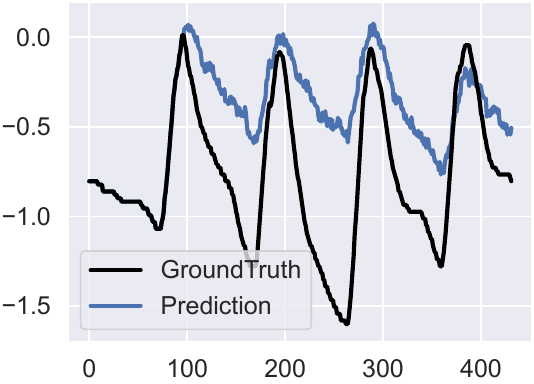}
}\hfill
\subfigure[ETTm2 snapshot 2.]{\includegraphics[width=0.24\linewidth]{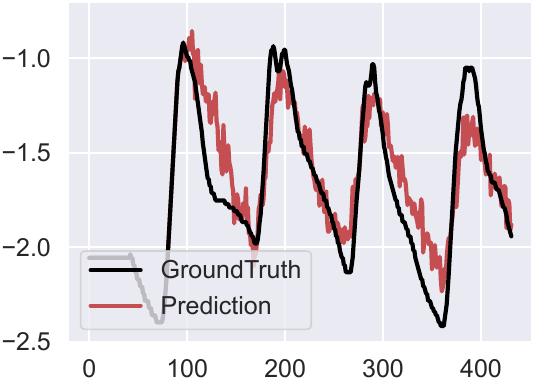}
\includegraphics[width=0.24\linewidth]{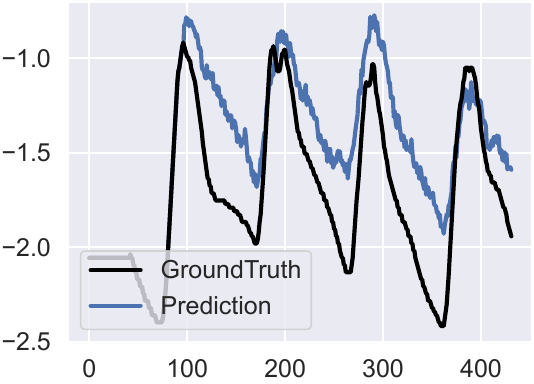}
}


\caption{The forecast sequences generated with DF and KMB-DF. The forecast length is set to 336 and the experiment is conducted on ETTm2.}
\label{fig:pred_app}
\end{center}
\end{figure}

\subsection{Generalization studies}\label{sec:generalize_app}
Additional experimental results of varying forecast models are available in \autoref{fig:backbone_app}, where four forecast models are involved on four datasets.

\begin{figure}[t]
\begin{center}
\includegraphics[width=0.245\linewidth]{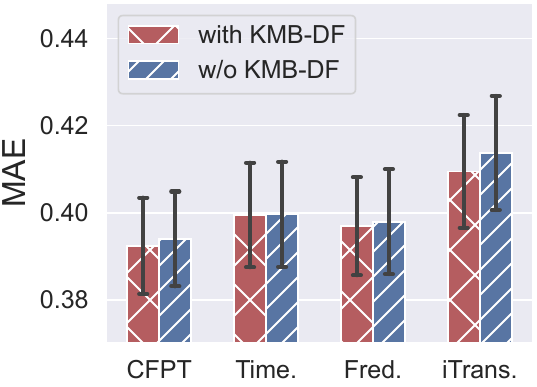}\hfill
\includegraphics[width=0.245\linewidth]{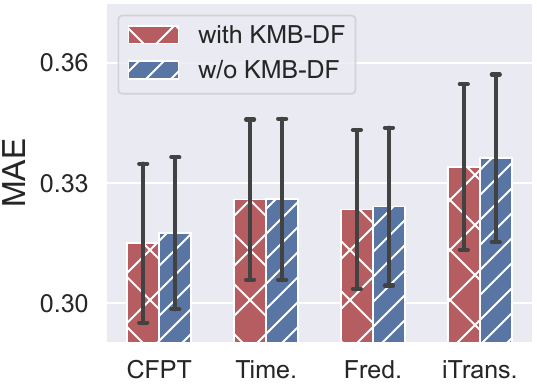}\hfill
\includegraphics[width=0.245\linewidth]{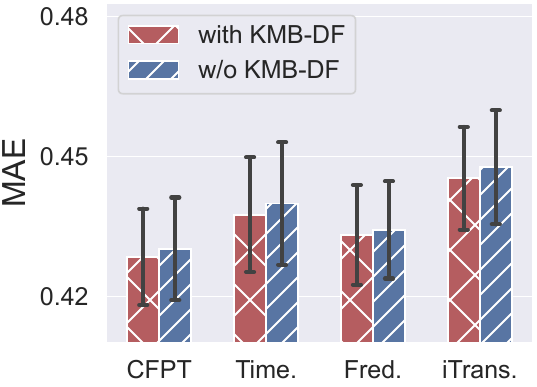}\hfill
\includegraphics[width=0.245\linewidth]{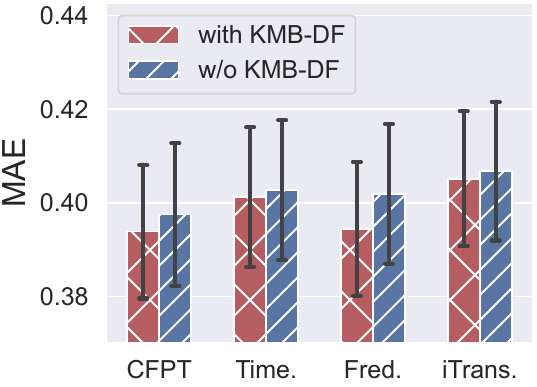}

\subfigure[ETTm1]{\includegraphics[width=0.245\linewidth]{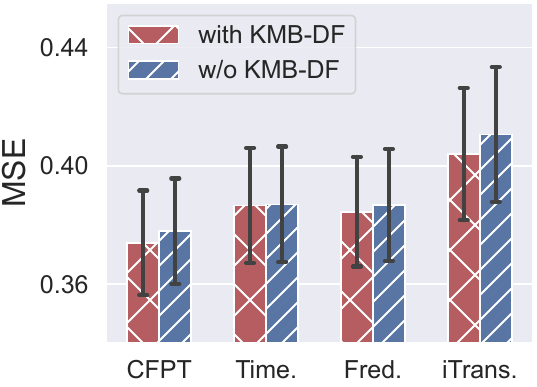}}\hfill
\subfigure[ETTm2]{\includegraphics[width=0.245\linewidth]{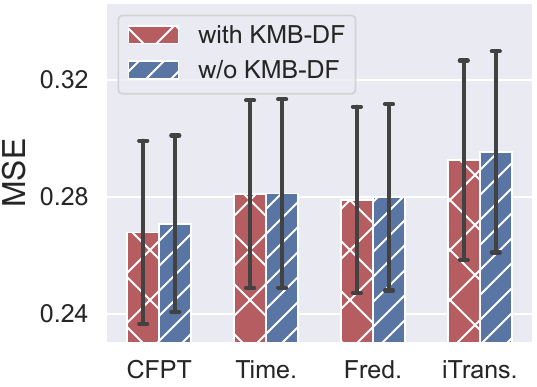}}\hfill
\subfigure[ETTh1]{\includegraphics[width=0.245\linewidth]{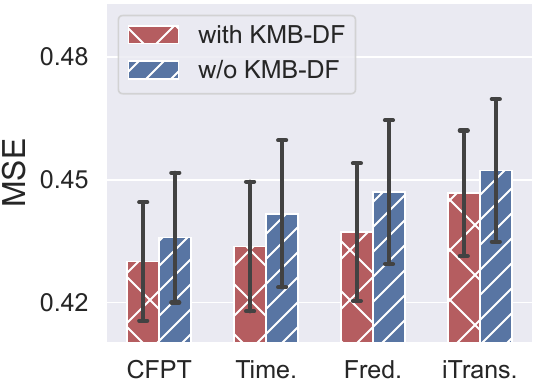}}\hfill
\subfigure[ETTh2]{\includegraphics[width=0.245\linewidth]{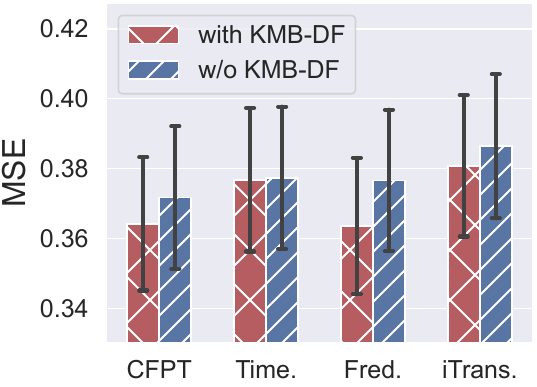}}
\caption{Performance of different forecast models with and without KMB-DF. The forecast errors are averaged over forecast lengths and the error bars represent 50\% confidence intervals.}
\label{fig:backbone_app}
\end{center}
\end{figure}

\subsection{Case study with PatchTST of varying historical lengths}
Additional experimental results of varying historical lengths are available in \autoref{tab:vary_seq_len}, complementing the fixed length of 96 used in the main text. The forecast models selected include CFPT~\citep{CFPT} which is the recent state-of-the-art forecast model, and PatchTST~\citep{PatchTST} which is known to require large historical lengths. The results demonstrate that KMB-DF consistently improves both forecast models across different history sequence lengths. 

\begin{table}[t]
\centering
\caption{Varying input sequence length results on the Weather dataset.}\label{tab:vary_seq_len}
\renewcommand{\arraystretch}{1}
\setlength{\tabcolsep}{12pt} \scriptsize
\centering
\renewcommand{\multirowsetup}{\centering}
\begin{tabular}{c|c|cc|cc|cc|cc}
    \toprule
    \multicolumn{2}{c|}{Models}
    & \multicolumn{2}{c}{\textbf{KMB-DF}} 
    & \multicolumn{2}{c|}{CFPT} 
    & \multicolumn{2}{c}{\textbf{KMB-DF}} & \multicolumn{2}{c}{PatchTST} \\
    \cmidrule(lr){3-4} \cmidrule(lr){5-6} \cmidrule(lr){7-8} \cmidrule(lr){9-10}
    \multicolumn{2}{c|}{Metrics}
    & MSE & MAE & MSE & MAE & MSE & MAE & MSE & MAE \\
    \midrule

\multirow{5}{*}{96}
& 96 & \bst{0.152} & \bst{0.196} & \subbst{0.156} & \subbst{0.201} & 0.178 & 0.219 & 0.189 & 0.230 \\
& 192 & \bst{0.204} & \bst{0.242} & \subbst{0.205} & \subbst{0.245} & 0.224 & 0.259 & 0.228 & 0.262 \\
& 336 & \bst{0.260} & \bst{0.284} & \subbst{0.261} & \subbst{0.286} & 0.280 & 0.299 & 0.288 & 0.305 \\
& 720 & \bst{0.339} & \bst{0.337} & \subbst{0.343} & \subbst{0.339} & 0.356 & 0.348 & 0.362 & 0.354 \\
\cmidrule(lr){2-10}
& Avg & \bst{0.239} & \bst{0.265} & \subbst{0.241} & \subbst{0.267} & 0.260 & 0.281 & 0.267 & 0.288 \\
\midrule
\multirow{5}{*}{192}
& 96 & \bst{0.148} & \bst{0.196} & \subbst{0.150} & \subbst{0.198} & 0.159 & 0.204 & 0.163 & 0.209 \\
& 192 & \bst{0.193} & \bst{0.238} & \subbst{0.195} & \subbst{0.239} & 0.206 & 0.247 & 0.207 & 0.249 \\
& 336 & \bst{0.247} & \bst{0.278} & \subbst{0.250} & \subbst{0.282} & 0.260 & 0.289 & 0.268 & 0.293 \\
& 720 & \bst{0.325} & \bst{0.331} & \subbst{0.330} & \subbst{0.333} & 0.335 & 0.339 & 0.338 & 0.339 \\
\cmidrule(lr){2-10}
& Avg & \bst{0.228} & \bst{0.261} & \subbst{0.231} & \subbst{0.263} & 0.240 & 0.270 & 0.244 & 0.273 \\
\midrule
\multirow{5}{*}{336}
& 96 & \bst{0.143} & \bst{0.192} & \subbst{0.145} & \subbst{0.195} & 0.153 & 0.202 & 0.158 & 0.208 \\
& 192 & \bst{0.188} & \subbst{0.235} & \subbst{0.188} & \bst{0.235} & 0.195 & 0.242 & 0.235 & 0.291 \\
& 336 & \bst{0.240} & \bst{0.275} & \subbst{0.240} & \subbst{0.277} & 0.249 & 0.283 & 0.252 & 0.287 \\
& 720 & \bst{0.317} & \bst{0.328} & \subbst{0.321} & \subbst{0.332} & 0.323 & 0.337 & 0.326 & 0.336 \\
\cmidrule(lr){2-10}
& Avg & \bst{0.222} & \bst{0.258} & \subbst{0.224} & \subbst{0.260} & 0.230 & 0.266 & 0.243 & 0.280 \\
\midrule
\multirow{5}{*}{720}
& 96 & \bst{0.144} & \bst{0.196} & \subbst{0.148} & \subbst{0.200} & 0.150 & 0.201 & 0.153 & 0.205 \\
& 192 & \bst{0.189} & \bst{0.238} & \subbst{0.189} & \subbst{0.240} & 0.196 & 0.245 & 0.205 & 0.254 \\
& 336 & \subbst{0.239} & \bst{0.279} & \bst{0.238} & \subbst{0.279} & 0.244 & 0.285 & 0.248 & 0.288 \\
& 720 & \bst{0.307} & \bst{0.327} & \subbst{0.310} & \subbst{0.329} & 0.313 & 0.335 & 0.317 & 0.339 \\
\cmidrule(lr){2-10}
& Avg & \bst{0.220} & \bst{0.260} & \subbst{0.221} & \subbst{0.262} & 0.226 & 0.266 & 0.231 & 0.272 \\

    \bottomrule
\end{tabular}
\end{table}

\subsection{Random Seed Sensitivity}
Additional experimental results of random seed sensitivity are available in  \autoref{tab:seed}, where we report the mean and standard deviation of results obtained from experiments conducted with five different random seeds (2022, 2023, 2024, 2025, and 2026). The results indicate minimal sensitivity of the proposed method to random initialization, as most averaged standard deviations remain below 0.005.

\begin{table}[t]
\centering
\caption{Experimental results ($\mathrm{mean}_{\pm\mathrm{std}}$) with varying seeds (2022-2026).}\label{tab:seed}
\renewcommand{\arraystretch}{1.2}
\setlength{\tabcolsep}{4pt}
\scriptsize
\centering
\renewcommand{\multirowsetup}{\centering}
\begin{tabular}{c|cc|cc|cc|cc}
    \toprule
    \rotatebox{0}{Dataset} 
    & \multicolumn{4}{c|}{ETTh1} 
    & \multicolumn{4}{c}{ETTm1} \\
    \cmidrule(lr){2-9}
    Models & \multicolumn{2}{c}{\textbf{KMB-DF}}
    & \multicolumn{2}{c|}{DF}
    & \multicolumn{2}{c}{\textbf{KMB-DF}} & \multicolumn{2}{c}{DF} \\
    \cmidrule(lr){2-3} \cmidrule(lr){4-5} \cmidrule(lr){6-7} \cmidrule(lr){8-9}
    Metrics & MSE & MAE & MSE & MAE & MSE & MAE & MSE & MAE \\
    \midrule

96 & 0.372$_{\pm 0.002}$ & 0.390$_{\pm 0.001}$ & 0.373$_{\pm 0.000}$ & 0.391$_{\pm 0.000}$ & 0.320$_{\pm 0.004}$ & 0.356$_{\pm 0.003}$ & 0.319$_{\pm 0.004}$ & 0.357$_{\pm 0.002}$ \\
192 & 0.425$_{\pm 0.001}$ & 0.421$_{\pm 0.001}$ & 0.427$_{\pm 0.000}$ & 0.421$_{\pm 0.000}$ & 0.354$_{\pm 0.002}$ & 0.380$_{\pm 0.001}$ & 0.359$_{\pm 0.005}$ & 0.382$_{\pm 0.001}$ \\
336 & 0.459$_{\pm 0.001}$ & 0.437$_{\pm 0.001}$ & 0.468$_{\pm 0.002}$ & 0.442$_{\pm 0.001}$ & 0.385$_{\pm 0.002}$ & 0.402$_{\pm 0.001}$ & 0.390$_{\pm 0.002}$ & 0.402$_{\pm 0.001}$ \\
720 & 0.462$_{\pm 0.013}$ & 0.463$_{\pm 0.006}$ & 0.475$_{\pm 0.003}$ & 0.467$_{\pm 0.002}$ & 0.443$_{\pm 0.001}$ & 0.434$_{\pm 0.000}$ & 0.446$_{\pm 0.000}$ & 0.435$_{\pm 0.000}$ \\
\cmidrule(lr){1-9}
Avg & 0.429$_{\pm 0.004}$ & 0.428$_{\pm 0.002}$ & 0.436$_{\pm 0.001}$ & 0.430$_{\pm 0.001}$ & 0.375$_{\pm 0.002}$ & 0.393$_{\pm 0.001}$ & 0.379$_{\pm 0.002}$ & 0.394$_{\pm 0.001}$ \\
    \bottomrule
\end{tabular}
\end{table}



\subsection{Complexity Analysis}\label{sec:comp_app}

We empirically evaluate the computational overhead of KMB-DF, as illustrated in \autoref{tab:complexity}. The experiments are conducted with a batch size of 128, variable dimension $\D=21$, and the number of balancing functions $\mathrm{K}=3$. As the forecast horizon $\T$ extends from 96 to 720, the running time for both forward and backward passes exhibits a marginal upward trend. This is expected, as a larger $\T$ increases the length of the joint sequences $\mathbf{z}$, thereby scaling the kernel matrix calculations required for moment balancing. Accordingly, forward FLOPs increase from 19.68M to 83.57M, backward FLOPs from 1.29M to 5.48M, and GPU memory consumption from 103.36 to 391.13 MiB. Nevertheless, the running time remains low even at $\T=720$, with the forward and backward passes taking 0.721 ms and 0.308 ms, respectively. During inference, the KMB-DF objective is removed, ensuring that the model's inference speed is unaffected. Consequently, KMB-DF enhances forecasting performance with minimal training cost and \textit{zero} additional inference latency.

\begin{table}[t]
\centering
\caption{Computational complexity of KMB-DF with varying forecast horizons.}
\label{tab:complexity}
\renewcommand{\arraystretch}{1.1}
\setlength{\tabcolsep}{4.4pt}
\scriptsize
\begin{tabular}{llllllll}
\toprule
Forecast Length & 96 & 128 & 192 & 256 & 336 & 512 & 720  \\
\midrule
Forward time (ms) & 0.670$_{\pm 0.002}$ & 0.672$_{\pm 0.003}$ & 0.679$_{\pm 0.002}$ & 0.686$_{\pm 0.003}$ & 0.689$_{\pm 0.002}$ & 0.686$_{\pm 0.004}$ & 0.721$_{\pm 0.004}$ \\
Backward time (ms) & 0.303$_{\pm 0.000}$ & 0.302$_{\pm 0.000}$ & 0.303$_{\pm 0.000}$ & 0.303$_{\pm 0.000}$ & 0.304$_{\pm 0.000}$ & 0.305$_{\pm 0.000}$ & 0.308$_{\pm 0.000}$ \\
Forward FLOPs & 19677832 & 22954504 & 29507848 & 36061192 & 44252872 & 62274568 & 83572936 \\
Backward FLOPs & 1291078 & 1506118 & 1936198 & 2366278 & 2903878 & 4086598 & 5484358 \\
GPU memory (MiB) & 103.358 & 120.338 & 146.096 & 180.057 & 224.559 & 294.280 & 391.127 \\
\bottomrule
\end{tabular}
\end{table}

\end{document}